\documentclass[oribibl]{llncs}

\usepackage{hyperref}
\usepackage{graphicx}
\usepackage{amsmath}
\usepackage{amssymb}

\renewcommand{\phi}{\varphi}

\renewcommand{\P}{\mathbb{P}}
\newcommand{\E}{\mathbb{E}}

\newcommand{\R}{\mathbb{R}}

\newcommand{\cA}{\mathcal{A}}

\renewcommand{\ts}{\mathrm{TS}}

\def\ds1{\mathds{1}}\renewcommand{\epsilon}{\varepsilon}

\newcommand{\argmax}{\mathop{\mathrm{argmax}}}

\newcommand{\eg}{e.g.}

\newcommand{\todo}[1]{\textbf{[TODO: #1]}}
\renewcommand{\todo}[1]{}

\renewcommand{\tilde}{\widetilde}

\newlength{\minipagewidth}
\newcommand{\citep}[1]{\cite{#1}}

\newcommand{\beq}{\begin{equation}}
\newcommand{\eeq}{\end{equation}}

\newcommand{\beqa}{\begin{eqnarray}}
\newcommand{\eeqa}{\end{eqnarray}}

\newcommand{\beqan}{\begin{eqnarray*}}
\newcommand{\eeqan}{\end{eqnarray*}}

\def\ba#1\ea{\begin{align*}#1\end{align*}} 
\def\banum#1\eanum{\begin{align}#1\end{align}} 

\newcommand{\reg}{\operatorname{R}}
\newcommand{\breg}{\operatorname{\bar{R}}}
\newcommand{\defeq}{:=}

\newcommand{\theModel}{$2$-Actions-And-$2$-Models}
\newcommand{\theNModel}{$2$-Actions-And-$N$-Models}
\newcommand{\theKNModel}{$K$-Actions-And-$N$-Models}

\newcommand{\secref}[1]{Section~\ref{#1}}
\newcommand{\figref}[1]{Figure~\ref{#1}}
\newcommand{\thmref}[1]{Theorem~\ref{#1}}
\newcommand{\lemref}[1]{Lemma~\ref{#1}}
\newcommand{\propref}[1]{Proposition~\ref{#1}}
\newcommand{\corref}[1]{Corollary~\ref{#1}}
\newcommand{\eqnref}[1]{Equation~\ref{#1}}

\newtheorem{assumption}{Assumption}

\title{On the Prior Sensitivity of Thompson Sampling}

\titlerunning{On the Prior Sensitivity of Thompson Sampling}

\author{Che-Yu Liu\inst{1}\thanks{Most of this work was done when C.Y. Liu was an intern at Microsoft.}%
\and Lihong Li\inst{2}}

\authorrunning{Liu and Li}

\institute{ORFE, Princeton University, Princeton, NJ, USA 08544 \\
\email{cheliu@princeton.edu}
\and
Microsoft Research, One Microsoft Way, Redmond, WA, USA 98052\\
\email{lihongli@microsoft.com}
}

\begin{document}

\maketitle

\begin{abstract}
The empirically successful Thompson Sampling algorithm for stochastic bandits has drawn much interest in understanding its theoretical properties.  One important benefit of the algorithm is that it allows domain knowledge to be conveniently encoded as a prior distribution to balance exploration and exploitation more effectively.  While it is generally believed that the algorithm's regret is low (high) when the prior is good (bad), little is known about the exact dependence.  This paper is a first step towards answering this important question: focusing on a special yet representative case, we fully characterize the algorithm's worst-case dependence of regret on the choice of prior.  As a corollary, these results also provide useful insights into the \emph{general} sensitivity of the algorithm to the choice of priors, when no structural assumptions are made.  In particular, with $p$ being the prior probability mass of the true reward-generating model, we prove $O(\sqrt{T/p})$ and $O(\sqrt{(1-p)T})$ regret upper bounds for the poor- and good-prior cases, respectively, as well as \emph{matching} lower bounds.  Our proofs rely on a fundamental property of Thompson Sampling and make heavy use of martingale theory, both of which appear novel in the Thompson-Sampling literature and may be useful for studying other behavior of the algorithm.  
\end{abstract}

\section{Introduction}
\label{sec:intro}

\emph{Thompson Sampling} (TS), also known as \emph{probability matching} and \emph{posterior sampling}, is a popular strategy for solving stochastic bandit problems.  An important benefit of this algorithm is that it allows domain knowledge to be conveniently encoded as a prior distribution to address the exploration-exploitation tradeoff more effectively.  In this paper, we focus on the sensitivity of the algorithm to the prior it uses.  In the rest of this section, we first define the bandit setting and notation, and describe Thompson Sampling; we will then discuss previous works that are most related to the present paper.

\subsection{Thompson Sampling for Stochastic Bandits}

In the multi-armed bandit problem, an agent is repeatedly faced with $K$ possible actions.   At each time step $t=1, \hdots ,T$, the agent chooses an action $I_t \in \cA\defeq\{1,\hdots ,K \}$, then receives reward $X_{I_t, t}\in\R$.  An \emph{eligible} action-selection strategy chooses actions at step $t$ based only on past observed rewards $\mathcal{H}_{t}= \{ I_s, X_{I_s, s}; 1\le s <t \}$  and potentially on an external source of randomness. More background on the bandit problem can be found in a recent survey~\citep{BC12}.

We make the following stochastic assumption on the underlying reward-generating mechanism.  Let $\Theta$ be a countable\footnote{Note that in this paper, we do not impose any continuity structure on the reward distributions $\nu(\theta)$ with respect to $\theta \in \Theta$. Therefore, it is easy to see that when $\Theta$ is uncountable, the (frequentist) regret of Thompson Sampling, as defined in \eqnref{eqn:freq-regret}, in the worst-case scenario is linear in time under most underlying models $\theta \in \Theta$.}   set of possible reward-generating models. When $\theta \in \Theta$ is the true underlying model, the rewards $(X_{i,t})_{t \geq 1}$ are i.i.d. random variables taking values in $[0,1]$ drawn from some {\em known} distribution $\nu_{i}(\theta)$ with mean $\mu_{i}(\theta)$. Of course, the agent knows neither the true underlying model nor the optimal action that yields the highest expected reward. The performance of the agent is measured by the regret incurred for not always selecting the optimal action. More precisely, the {\em frequentist regret} (or \emph{regret} for short) for an eligible action-selection strategy $\pi$ under a certain reward-generating model $\theta$ is defined as
\vspace{-2mm}
\begin{equation}
\reg_T(\theta, \pi) \defeq \E \sum_{t=1}^T \left(\max_{i \in \cA} \mu_{i}(\theta) - \mu_{I_t}(\theta) \right)\,, \label{eqn:freq-regret}
\end{equation}
where the expectation is taken with respect to the rewards $(X_{i,t})_{i \in \cA,t \geq 1}$, generated according to the model $\theta$, and the potential external source of randomness.

If one imposes a prior distribution $p$ over $\Theta$, then it is natural to consider the following notion of average regret known as {\em Bayes regret}:
\vspace{-2mm}
\begin{equation}
\breg_T(\pi) \defeq \E_{\theta \sim p} \reg_T(\theta,\pi) = \sum_{\theta \in \Theta} \reg_T(\theta,\pi) p(\theta)\,. \label{eqn:bayes-regret}
\end{equation}

The Thompson Sampling strategy was proposed in probably the very first paper on multi-armed bandits~\citep{Tho33}. This strategy takes as input a prior distribution $p_1$ for $\theta \in \Theta$. At each time $t$, let $p_t$ be the posterior distribution for $\theta$ given the prior $p_1$ and the history $\mathcal{H}_{t}= \{ I_s, X_{I_s, s};1\le s<t\}$. Thompson Sampling selects an action randomly according to its posterior probability of being the optimal action. Equivalently, Thompson Sampling first draws a model $\theta_t$ from $p_t$ (independently from the past given $p_t$) and it pulls $I_t \in \argmax_{i \in \cA} \mu_{i}(\theta_t)$.  For concreteness, we assume that the distributions $(\nu_{i}(\theta))_{i \in \cA, \theta \in \Theta}$ are absolutely continuous with respect to some common measure $\nu$ on $[0,1]$ with likelihood functions $(\ell_{i}(\theta)(\cdot))_{i \in \cA, \theta \in \Theta}$. The posterior distributions $p_t$ can be computed recursively by Bayes rule as follows:
$$ p_{t+1}(\theta) = \frac{ p_{t}(\theta)  \ell_{I_t}(\theta)(X_{I_t,t}) } { \sum_{\eta \in \Theta} p_{t}(\eta)  \ell_{I_t}(\eta)(X_{I_t,t})   } . $$ 
We denote by $\ts(p_1)$ the Thompson Sampling strategy with prior $p_1$.

Two remarks are in order.  First, the setup above is a \emph{discretized} version of rather general bandit problems.  For example, the $K$-armed bandit is a special case, where $\Theta$ is the Cartesian product of the sets of reward distributions of all arms.  As another example, in linear bandits~\citep{Abbasi11Improved,Chu11Contextual}, $\Theta$ is a set of candidate coefficient vectors that determine the expected reward function.  Discretization of $\Theta$ provides a convenient yet useful approximation that leads to simplicity in expositions and analysis.  Such an abstract formulation is analogous to the expert setting widely studied in the online-learning literature~\citep{CL06}; also see a recent study of Thompson Sampling with $2$ and $3$ experts~\citep{Gravin16Towards}.

Second, although we assume reward are bounded, 
some results in the paper, especially \lemref{lem:martingale} that may be of independent interest, still hold with unbounded rewards.
\todo{The lemma refers to the appendix.}

\subsection{Related Work}

Recently, Thompson Sampling has gained a lot of interest, largely due to its empirical successes~\citep{CLi11,Graepel10Web,May12Optimistic,Scott10Modern}.  Furthermore, this strategy is often easy to be combined with complex reward models and easy to implement~\citep{Gopalan14Thompson,Komiyama15Optimal,Xia15Thompson}.  While asymptotic, no-regret results are known~\citep{May12Optimistic}, these empirical successes inspired finite-time analyses that deepen our understanding of this old strategy.

For the classic $K$-armed bandits, regret bounds comparable to the the more widely studied UCB algorithms are obtained~\citep{AG12,AG13,KKM12,Honda14Optimality}, matching a well-known asymptotic lower bound~\citep{LR85}.  For linear bandits of dimension $d$, an $\tilde{O}(d\sqrt{TK})$ upper bound has been proved~\citep{AG13b}.  
All these bounds, while providing interesting insights about the algorithm,
assume {\em non-informative priors} (often {\em uniform priors}), and essentially show that Thompson Sampling has a comparable regret to other popular strategies, especially those based on upper confidence bounds. Unfortunately, the bounds do \emph{not} show what role prior plays in the performance of the algorithm.
In contrast, a \emph{variant} of Thompson Sampling is proposed, with a bound that depends explicitly on the entropy of the prior~\citep{Li13Generalized}.  However, their bound has an $O(T^{2/3})$ dependence on $T$ that is likely sub-optimal.

Another line of work in the literature focuses on the {\em Bayes regret} with an {\em informative prior}.  Previous work has shown that, for any prior in the two-armed case, TS is a $2$-approximation to the optimal strategy that minimizes the ``stochastic'' (Bayes) regret~\citep{GM14}.  It has also been shown that in the K-armed case, the Bayes regret of TS is always upper bounded by $O(\sqrt{KT})$ for any prior~\citep{BL13,Russo14Learning}. These results were later improved~\citep{Russo14Information} to a prior-dependent bound $O(\sqrt{H(q)KT})$ where $q$ is the prior distribution of the optimal action, defined as $q(i)=\P_{\theta \sim p_1} ( i= \argmax_{j \in \cA} \mu_{j}(\theta) ) $,  and  $ H(q)=-\sum_{i=1}^K q(i) \log q(i)$ is the entropy of $q$. While this bound elegantly quantifies, in terms of \emph{averaged} regret, how Thompson Sampling exploits prior distributions, it does not tell how well Thompson Sampling works in \emph{individual} problems.  Indeed, in the analysis of Bayes regret, it is unclear what a ``good'' prior means from a theoretical perspective, as the definition of Bayes regret essentially assumes the prior is correctly specified.  In the extreme case where prior $p_1$ is a point mass, $H(q)=0$ and the Bayes regret is trivially $0$.

To the best of our knowledge, {\em our work is the first to consider frequentist regret of Thompson Sampling with an informative prior.}  Specifically, we focus on understanding TS's sensitivity to the choice of prior, making progress towards a better understanding of such a popular Bayesian algorithm.  It is shown that, while a strong prior can lower the Bayes regret substantially~\citep{Russo14Information}, such a benefit comes with a cost: if the true model happens to be assigned a low prior (the poor-prior case), the frequentist regret will be very large, which is consistent with a recent result on Pareto regret frontier~\citep{Lattimore15Pareto}.  Our findings suggest Thompson Sampling can be \emph{under}-exploring in general.  Techniques
like those in the ``mini-monster'' algorithm~\citep{Agarwal14Taming} may be necessary to modify Thompson Sampling to make it less prior-sensitive.  It is an open question whether such modified Thompson Sampling algorithms can still take advantage of an informative prior to enjoy a small Bayes regret.

Finally, our analysis makes critical use of a certain martingale property of Thompson Sampling.  
Although martingales have been applied to hypothesis testing, for example, in analyzing the statistical behavior of likelihood ratios~\citep{Bartroff13Sequential}, our use of martingales to analyze the behavior of posteriors in TS is new, to the best of our knowledge.  Moreover, a different martingale property was used by other authors to study the Bayesian multi-armed bandit problem, where the reward at the current ``state'' is the same as the expected reward over the distribution of the next state when a play is made in the current state~\citep{GM13,GM14}. Their martingale property is different from ours: their martingales apply to the reward at the current state, while ours refers to the inverse of the posterior probability mass of the true model (see \secref{sec:prelim} for details).

\todo{Hypothesis testing.}
\todo{Read previous reviews and rebuttals.}
\todo{Cite more recent work, both applied and theoretical.}

\section{Main Results}
\label{sec:main-results}

Naturally, we expect the regret of Thompson Sampling to be small when the true reward-generating model is given a large prior probability mass, and vice versa.  An interesting and important question is to understand the sensitivity of the algorithm's regret to the prior it takes as input.  We take a  minimalist approach, and investigate a special yet meaningful case. 
Our results fully characterize the worst-case dependence of TS's regret on the prior, which also provides important insights into a more general case as a corollary.  Furthermore, our analysis appears novel to the best of our knowledge, making heavy use of martingale techniques to analyze the behavior of the posterior probability.  Such techniques may be useful for studying other bandit algorithms.

Similar to the expert setting~\citep{CL06}, we assume access to a set of candidate models, $\Theta=\{\theta_1,\theta_2,\ldots,\theta_N\}$ for $N\ge 2$.  This setting is referred to as \textbf{\theKNModel}, where $K$ is the cardinality of the action set.  For simplicity, in this work, we restrict ourselves to the binary action case: $K=2$.  
Finally, the special case with $N=2$ and $K=2$ is called \textbf{\theModel}.

Two comments are in order.  First, our goal in this work is \emph{not} to solve these specialized bandit problems, but rather to understand prior sensitivity of TS.  Such seemingly simplistic problems happen to be nontrivial enough to be useful in our constructive proof of matching lower bounds.  Second, we aim to understand TS's prior sensitivity \emph{without} making any structural assumptions about $\Theta$.  A natural next step of this work is to investigate, with a structural $\Theta$ (\eg, linear), how robust TS is to the prior.


Our upper-bound analysis requires the following smoothness assumption of the likelihood functions of models in $\Theta$.  Note that this assumption is needed only in the upper-bound analysis, but \emph{not} in the lower-bound proofs.

\begin{assumption}(Smoothness) \label{ass:smooth}
There exists constant $s > 1$ such that $\nu$-almost surely, for $i\in\{1,2\}$,
$s^{-1} \cdot \ell_i(\theta_1) \leq \ell_i(\theta_2) \leq s \cdot \ell_i(\theta_1)$.
\end{assumption}

\begin{remark}
While this assumption does not hold for all distributions, it holds for some important ones, such as Bernoulli distributions $Bern(p)$ with mean $p\in(0,1)$.  On one hand, the assumption essentially avoids situations where a single application of Bayes rule can change posteriors by too much, analogous to bounded gradients or rewards in most online-learning literature. On the other hand, a small $s$ value in the assumption tends to create hard problems for Thompson Sampling, since models are less distinguishable. Therefore, the assumption does \emph{not} trivialize the problem.
\end{remark}

The first main result of this paper is the following upper bound; see \secref{sec:ub} for more details:

\begin{theorem}  \label{thm:ub}
Consider the \theModel\ case and assume that Assumption~\ref{ass:smooth} holds. Then, the regret of Thompson Sampling with prior $p_1$ satisfies
$\mathrm{R}_T(\theta_1, \ts(p_1)) = O(s\sqrt{T/p_1(\theta_1)})$.
Moreover, when $p_1(\theta_1) \geq 1-\frac{1}{8s^2}$, we have 
$\mathrm{R}_T(\theta_1, \ts(p_1)) = O(s^4\sqrt{(1-p_1(\theta_1))T})$.
\end{theorem}

\begin{remark}  \label{rmk:ub}
The above upper bounds have the same dependence on $T$ and $p_1(\theta_1)$ as the lower bounds to be given in Theorems~\ref{thm:lb-poor} and \ref{thm:lb-good} below.  Moreover, both bounds are increasing functions of the smoothness parameter $s$. Because problems with small $s$ tend to be harder for Thompson Sampling, our upper bounds are tight up to a universal constant for a fairly general class of hard problems. We conjecture that the dependence on $s$ is an artifact of our proof techniques and can be removed to get tighter upper bounds for all problem instances of the \theModel\ case.  
\end{remark}

The next two theorems give \emph{matching} lower bounds for the poor- and good-prior cases, respectively.  More details are given in \secref{sec:lb}.

\begin{theorem} \label{thm:lb-poor}
Consider the \theModel\ case. Let $p_1$ be a prior distribution and $T \geq \frac{1}{p_1(\theta_1)}$. Consider the following specific problem instance:
$\nu_1(\theta_1)=Bern\left(\frac{1}{2}+\Delta\right)$, $\nu_1(\theta_2)=Bern\left(\frac{1}{2}-\Delta \right)$,
$\nu_2(\theta_1)=\nu_2(\theta_2)=Bern\left(\frac{1}{2}\right)$,
where $\Delta = 1/\sqrt{8p_1(\theta_1)T}$.
Then, the regret of Thompson Sampling with prior $p_1$ satisfies the following: if $p_1(\theta_1)\leq \frac{1}{2}$, then
$\mathrm{R}_T(\theta_1, TS(p_1)) \geq \frac{1}{168\sqrt{2}} \sqrt{ \frac{T}{p_1(\theta_1)}}$.
\end{theorem}

\begin{theorem} \label{thm:lb-good}
Consider the \theModel\ case. Let $p_1$ be a prior distribution and $T \geq \frac{1}{1-p_1(\theta_1)}$. Consider the following specific problem instance with Bernoulli reward distributions:
$\nu_1(\theta_1)=\nu_1(\theta_2)=Bern\left(\frac{1}{2}\right)$, $\nu_2(\theta_1)=Bern\left(\frac{1}{2}-\Delta\right)$, 
$\nu_2(\theta_2)=Bern\left(\frac{1}{2}+\Delta \right)$,
where $\Delta = \sqrt{\frac{1}{8(1-p_1(\theta_1))T}}$.
Then the regret of Thompson Sampling with prior $p_1$ satisfies 
$\mathrm{R}_T(\theta_1, TS(p_1)) \geq \frac{1}{10\sqrt{2}}  \sqrt{(1- p_1(\theta_1))T}$.
\end{theorem}

The lower bounds in the \theModel\ case easily imply the lower bounds in the \emph{general} case. 
\begin{corollary}(General Lower Bounds)  \label{cor:general_lb}
Consider the case with two actions and an arbitrary countable $\Theta$. Let $p_1$ be a prior over $\Theta$ and $\theta^* \in \Theta$ be the true model. Then, there exist problem instances where the regrets of Thompson Sampling are $\Omega(\sqrt{\frac{T}{p_1(\theta^* )}})$ and $\Omega(\sqrt{(1-p_1(\theta^* ))T})$ for small $p_1(\theta^* )$ and large  $p_1(\theta^* )$, respectively. 
\end{corollary}

\begin{remark}
These lower bounds show that the performance of Thompson Sampling can be quite sensitive to the choice of input prior, especially when the prior is poorly chosen.
\end{remark}

Due to space limit, we can only include the more important, novel or challenging parts of the analysis in the paper.  A complete proof, together with simulation results corroborating our theoretical findings, are given in a full version~\citep{Liu15Prior}.


\subsection{Comparison to Previous Results}
\label{sec:comparison}

Note that an upper bound in the \theKNModel\ case can be derived from an earlier result~\citep{Russo14Information}, which upper-bounds the Bayes regret, $\breg_T(TS(p_1))$:
$$ \reg_T(\theta_1 , TS(p_1)) \leq \frac{ \breg_T(TS(p_1))}{p_1(\theta_1 )} = O\left( \frac{\sqrt{H(q)KT}}{p_1(\theta_1 )} \right)\,, $$
where $\theta_1\in\Theta$ is the unknown, true model. On one hand, in the \theModel\ case, 
the above upper bound becomes $O\left( \sqrt{\log \left( \frac{1}{p_1(\theta_1)} \right) \frac{T}{p_1(\theta_1)}} \right)$ for small $p_1(\theta_1)$, and $O\left( \sqrt{\log \left( \frac{1}{1-p_1(\theta_1)} \right) (1-p_1(\theta_1))T} \right)$ for large $p_1(\theta_1)$. Our upper bounds in Theorem~\ref{thm:ub} remove the extraneous logarithmic terms in these upper bounds. On the other hand, the above general upper bound can be further upper bounded by $O\left( \frac{\sqrt{T}}{p_1(\theta_1 )} \right)$ for small $p_1(\theta_1 )$ and  $O\left( \sqrt{\log \left( \frac{1}{1-p_1(\theta_1)} \right) (1-p_1(\theta_1))T} \right)$ for large $p_1(\theta_1)$. We conjecture that these general upper bounds can be improved to match our lower bounds in Corollary~\ref{cor:general_lb}, especially for small $p_1(\theta_1)$.  But it remains open how to extend our proof techniques for the \theModel\ case to get tight general upper bounds.

It is natural to compare Thompson Sampling to exponentially weighted algorithms, a well-known family of algorithms that can also take advantage of prior knowledge. If we see each model $\theta \in \Theta$ as an expert who recommends the optimal action based on distributions specified by $\theta$, and use the prior $p_1$ as the initial weights assigned to the experts, then the EXP4 algorithm~\citep{ACFS03} has a regret of
$O\left(KT \gamma + \frac{1}{\gamma}\log\frac{1}{p_1(\theta^*)}\right)$,
with a parameter $\gamma\in(0,1)$.
For the sake of simplicity, we only do the comparison in the \theModel\ case. By trying to match or even beat the upper bounds in Theorem~\ref{thm:ub}, we reach the choice that $\gamma = \sqrt{H(p_1)/T}$. Assuming that  $\theta_1$ is the true model, the bound becomes 
 $O\left( \sqrt{\log \left( \frac{1}{p_1(\theta_1)} \right) \frac{T}{p_1(\theta_1)}} \right)$ for small $p_1(\theta_1)$, and  $O\left( \sqrt{\log \left( \frac{1}{1-p_1(\theta_1)} \right) (1-p_1(\theta_1))T} \right)$ for large $p_1(\theta_1)$. Thus, although EXP4 is not a Bayesian algorithm, it has the same worst-case dependence on prior as Thompson Sampling, up to logarithmic factors. This is partly explained by the fact that such algorithms are designed to perform well in the worst-case (adaptive adversarial) scenario.  On the contrary, by design, Thompson Sampling takes advantage of prior information more efficiently in most cases, especially when there is certain structure on the model space $\Theta$~\citep{BL13}. Note that in this paper, we do not impose any structure on $\Theta$, thus our lower bounds do not contradict existing results in the literature with non-informative priors (where $p(\theta^*)$ can be very small as $\Theta$ is typically large).

Finally, our proof techniques are new in the Thompson Sampling literature, to the best of our knowledge. The key observation is that the inverse of the posterior probability of the true underlying model is a martingale (\lemref{lem:martingale}). It allows us to use results and techniques from martingale theory to quantify the time and probability that the posterior distribution hits a certain threshold. Then, the regret of Thompson Sampling can be analyzed  separately before and after hitting times.

\section{Preliminaries}
\label{sec:prelim}

In this section, we study a fundamental martingale property of Thompson Sampling and its implications.  The results are essential to proving our upper bounds in \secref{sec:ub}.  Note that a similar property holds for posterior updates using Bayes rule, which however does not involve action selection.

Throughout this paper, for a random variable $Y$, we will use the shorthand $\E_t[Y]$ for the conditional expectation $\E[ Y | \mathcal{H}_{t} ]$. Moreover, we denote by $\E^{\theta}[Y]$ the expectation of $Y$ when $\theta$ is the true underlying model, i.e., when $X_{i,t}$ has distribution $\nu_i(\theta)$. The notation $\P^{\theta}[\cdot]$ is similarly defined.  Furthermore, we use the shorthand $a \wedge b$ for $\min\{a,b\}$.

\begin{lemma}{(Martingale Property)} \label{lem:martingale}
Assume that $\Theta$ is countable and that $\theta^* \in \Theta$ is the true reward-generating model.  Then, the stochastic process $( p_{t}(\theta^*)^{-1} )_{t \geq 1}$ is a martingale with respect to the filtration $(\mathcal{H}_{t})_{t \geq 1}$. 
\end{lemma}

\begin{proof}
First, recall that conditioned on $\mathcal{H}_{t}$, $p_t$ is deterministic. Then one has
\begin{eqnarray*}
\E_t^{\theta^*} [ p_{t+1}(\theta^*)^{-1} ]  &=& 
\E_t^{\theta^*} \left[ \frac{ \sum_{\eta \in \Theta} p_{t}(\eta) \ell_{I_t}(\eta)(X_{I_t,t})   }{ p_{t}(\theta^*)  \ell_{I_t}(\theta^*)(X_{I_t,t}) }   \right]  \\
  &=&  \sum_{i=1}^K  \P_t^{\theta^*}(I_t=i)  \E_t^{\theta^*} \left[ \frac{ \sum_{\eta \in \Theta} p_{t}(\eta)  \ell_{i}(\eta)(X_{i,t})   }{ p_{t}(\theta^*)  \ell_{i}(\theta^*)(X_{i,t}) } \right]   \\  
  &=&  \sum_{i=1}^K  \P_t^{\theta^*}(I_t=i)  \int   \frac{ \sum_{\eta \in \Theta} p_{t}(\eta) \ell_{i}(\eta)(x)   }{ p_{t}(\theta^*)  \ell_{i}(\theta^*)(x) } \ell_{i}(\theta^*)(x) \, \mathrm{d}\nu(x) \\ 
  &=& p_{t}(\theta^*)^{-1} \sum_{i=1}^K  \P_t^{\theta^*}(I_t=i) \int    \sum_{\eta \in \Theta} p_{t}(\eta)  \ell_{i}(\eta)(x) \,  \mathrm{d}\nu(x) \\ 
  &=& p_{t}(\theta^*)^{-1} \sum_{i=1}^K  \P_t^{\theta^*}(I_t=i) \, = \, p_{t}(\theta^*)^{-1} \,,
\end{eqnarray*} 
where the second last equality follows from the fact that $ \int  \ell_{i}(\eta)(x) \,  \mathrm{d}\nu(x) =1 $ for any $\eta \in \Theta$.  \qed
\end{proof}

Consider the \theModel\ case. Let $A,B \in (0,1)$ be two constants such that $A > p_1(\theta_1) > B$.  We define the following \emph{hitting time}s and \emph{hitting probabilities}:
$\tau_A = \inf \{t \geq 1, p_t(\theta_1) \geq A \}$,
$\tau_B = \inf \{t \geq 1, p_t(\theta_1) \leq B \}$,
$q_{A,B} = \P^{\theta_1}(\tau_A < \tau_B)$, and
$q_{B,A} = \P^{\theta_1}(\tau_A > \tau_B)$.
%
%
The martingale property above implies the following  results which will be used repeatedly in the proofs of our results. 
\begin{lemma} \label{lem:stop}
Consider the \theModel\ case with $\Delta>0$, where $\Delta$ is as defined in \thmref{thm:lb-poor}.  Then, we have $\tau_A < +\infty$ almost surely.  Furthermore, assume that $\tau_B<+\infty$ and that there exists constant $\gamma > 0$ so that $p_{\tau_B}(\theta_1) \geq \gamma $ almost surely, then 
\begin{align*}
q_{A,B} &= \frac{ \E^{\theta_1}[p_{\tau_B}(\theta_1)^{-1} | \tau_A > \tau_B] - p_1(\theta_1)^{-1} }{ \E^{\theta_1}[p_{\tau_B}(\theta_1)^{-1} | \tau_A > \tau_B ] - \E^{\theta_1}[p_{\tau_A}(\theta_1)^{-1} | \tau_A < \tau_B] } \,\, \text{ and }
\end{align*}
\begin{align*}
q_{B,A} &= \frac{ p_1(\theta_1)^{-1} - \E^{\theta_1}[p_{\tau_A}(\theta_1)^{-1}| \tau_A < \tau_B] }{ \E^{\theta_1}[p_{\tau_B}(\theta_1)^{-1} | \tau_A > \tau_B] - \E^{\theta_1}[p_{\tau_A}(\theta_1)^{-1} | \tau_A < \tau_B] }\,.
\end{align*}
Finally, $q_{B,A} \leq \frac{B}{p_1(\theta_1)}$ and $q_{B,A} \leq \frac{1-p_1(\theta_1)}{A-B} $.
\end{lemma}

\begin{proof}
We first argue that $\tau_A < +\infty$ almost surely. Define the event $E=\{\tau_A= +\infty \}$. Under the event $E$, $p_t(\theta_1)$ is always upper bounded by $A$ for any $t$. Thus 
$$ \mathrm{R}_T(\theta_1, TS(p_1)) = \Delta \cdot \E^{\theta_1}  \sum_{t=1}^T  p_t(\theta_2) \geq  \P^{\theta_1}(E)\Delta(1-A)T. $$ 
It follows that  $$ \breg_T(TS(p_1)) \geq p_1(\theta_1)\reg_T(\theta_1 , TS(p_1)) \geq p_1(\theta_1) \P^{\theta_1}(E)\Delta(1-A)T .   $$ 
However, it was proven~\citep{BL13} that the Bayes risk $\breg_T(TS(p_1)) $ is always upper bounded by $O(\sqrt{T})$. Therefore we must have $ \P^{\theta_1}(E) = 0$; that is $\tau_A < +\infty$ almost surely. This implies that $p_{\tau_A \wedge \tau_B  }(\theta_1)$ is well defined and $q_{A,B}+q_{B,A}=1$.

Now, by \lemref{lem:martingale}, $( p_{t}(\theta_1)^{-1} )_{t \geq 1}$ is a martingale.  It is easy to verify that $\tau_A$ and $\tau_B$ are both stopping times with respect to the filtration $(\mathcal{H}_{t})_{t \geq 1}$. Then it follows from Doob's optional stopping theorem that for any $t$, $ \E^{\theta_1}[ p_{t \wedge \tau_A \wedge \tau_B  }(\theta_1)^{-1}] = p_1(\theta_1)^{-1}$. Moreover, for any $t \geq 1$, $p_{t \wedge \tau_A \wedge \tau_B  }(\theta_1)^{-1} \leq  \gamma^{-1}$ (Note that by definition, $\gamma \leq B$). Hence, by Lebesgue's dominated convergence theorem, $ \E^{\theta_1}[ p_{t \wedge \tau_A \wedge \tau_B  }(\theta_1)^{-1}] \longrightarrow \E^{\theta_1}[ p_{\tau_A \wedge \tau_B  }(\theta_1)^{-1}]  $  as $t \rightarrow +\infty$.  Thus,
\begin{eqnarray*}
p_1(\theta_1)^{-1} &=& \E^{\theta_1}[ p_{\tau_A \wedge \tau_B  }(\theta_1)^{-1}] \\
&=& q_{A,B} \E^{\theta_1}[ p_{\tau_A  }(\theta_1)^{-1} | \tau_A < \tau_B] +  q_{B,A} \E^{\theta_1}[ p_{\tau_B  }(\theta_1)^{-1} | \tau_A > \tau_B] \,.
\end{eqnarray*}
The above equality combined with $q_{A,B}+q_{B,A}=1$ gives the desired expressions for $q_{A,B}$ and $q_{B,A}$. Finally, we have 
\begin{eqnarray*}
q_{B,A} &=& \frac{ p_1(\theta_1)^{-1} - \E^{\theta_1}[p_{\tau_A}(\theta_1)^{-1}| \tau_A < \tau_B] }{ \E^{\theta_1}[p_{\tau_B}(\theta_1)^{-1} | \tau_A > \tau_B] - \E^{\theta_1}[p_{\tau_A}(\theta_1)^{-1} | \tau_A < \tau_B] } \\
&\leq &   \frac{ p_1(\theta_1)^{-1} }{ \E^{\theta_1}[p_{\tau_B}(\theta_1)^{-1} | \tau_A < \tau_B ]} \leq \frac{B}{p_1(\theta_1)}  \,
\end{eqnarray*}
and 
\begin{eqnarray*}
q_{B,A} &=& \frac{ p_1(\theta_1)^{-1} - \E^{\theta_1}[p_{\tau_A}(\theta_1)^{-1}| \tau_A < \tau_B] }{ \E^{\theta_1}[p_{\tau_B}(\theta_1)^{-1} | \tau_A > \tau_B] - \E^{\theta_1}[p_{\tau_A}(\theta_1)^{-1} | \tau_A < \tau_B] } \\
&\leq& \frac{ p_1(\theta_1)^{-1} - 1}{ B^{-1} - A^{-1} } 
= \frac{AB}{p_1(\theta_1)} \frac{1-p_1(\theta_1)}{A-B}  \leq \frac{1-p_1(\theta_1)}{A-B} \,.
\end{eqnarray*}
\qed
\end{proof}

\section{Upper Bounds}
\label{sec:ub}

In this section, we focus on the \theModel\ case. We present and prove our results on the upper bounds for the frequentist regret of Thompson Sampling.  Due to space limitation, we only sketch the proof for the poor-prior case (first part of \thmref{thm:ub}); complete proofs, including those for the good-prior case, will appear in a long version.

We start with a simple lemma that follows immediate from Assumption~\ref{ass:smooth}:

\begin{lemma} \label{lem:smooth}
Under Assumption~\ref{ass:smooth}, regardless of either $\theta_1$ or $\theta_2$ being the true underlying model,  for any $\theta \in \{\theta_1, \theta_2\}$, 
$s^{-1}\cdot p_{t}(\theta) \leq  p_{t+1}(\theta) \leq s \cdot p_{t}(\theta) \,\,\, \nu \text{-almost surely}$. 
\end{lemma}

The next lemma describes how the posterior probability mass of the true model evolves over time.  It can be proved by direct, although a bit tedious, calculations.
\begin{lemma} \label{lem:tech-2b2}
Consider the \theModel\ case. We have the following inequalities concerning various functionals of the stochastic process $(p_t(\theta_1))_{t\geq 1}$.
\begin{description}
\item{\textbf{(a)} For $t \geq 1$, $\E_t^{\theta_1} \left[ \log(p_t(\theta_1)^{-1}) - \log(p_{t+1}(\theta_1)^{-1}) \right]$ \\ $\ge \frac{1}{2} \sum_{i \in \{1,2\} }  p_t(\theta_i)p_t(\theta_2)^2  |\mu_i(\theta_1)-\mu_i(\theta_2)|^2$. \vspace{1mm}}
%
%
\item{\textbf{(b)} For $t \geq 1$, $\E^{\theta_1}[p_{t+1}(\theta_1)] \geq \E^{\theta_1}[p_t(\theta_1)]$ and \\ $\E_t^{\theta_1} \left[ p_{t+1}(\theta_1) - p_t(\theta_1)  \right]   \leq  \sum_{i \in \{1,2\} } p_t(\theta_i)  p_{t}(\theta_1)  p_{t}(\theta_2)     \E^{\theta_1}\left[  \frac{\ell_{i}(\theta_1)(X_{i,t}) }{\ell_{i}(\theta_2)(X_{i,t}) }  - 1 \right]$.}
%
\item{\textbf{(c)} For $t \geq 1$, $\E_t^{\theta_1} \left[ (1-p_{t+1}(\theta_1))^{-1} - (1-p_t(\theta_1))^{-1} \right]$ \\
$=\sum_{i \in \{1,2\} } p_t(\theta_i) \frac{p_{t}(\theta_1)}{p_{t}(\theta_2)}  \E^{\theta_1}\left[  \frac{  \ell_{i}(\theta_1)(X_{i,t}) }{  \ell_{i}(\theta_2)(X_{i,t})  } - 1  \right]$ \\
$\geq \frac{p_{t}(\theta_1)^2}{2p_{t}(\theta_2)} |\mu_1(\theta_1)-\mu_1(\theta_2)|^2+   \frac{p_t(\theta_1)}{2} |\mu_2(\theta_1)-\mu_2(\theta_2)|^2$. \vspace{1mm}}
\item{\textbf{(d)} $\mathrm{R}_T(\theta_1, TS(p_1)) \leq \Delta T (1- p_1(\theta_1)) $.}
\end{description}
\end{lemma}

We now introduce some notation. Let $\Delta=\mu_1(\theta_1)-\mu_2(\theta_1)$, $\Delta_1=|\mu_1(\theta_1)-\mu_1(\theta_2)|$ and $\Delta_2=|\mu_2(\theta_1)-\mu_2(\theta_2)|$. Obviously, $\Delta \leq \Delta_1+\Delta_2$.  We assume $\Delta>0$ to avoid the generated case. To simplify notation, define the regret function $\mathrm{R}_T(\cdot)$ by $\mathrm{R}_T(p_1(\theta_1)) = \mathrm{R}_T(\theta_1, \ts(p_1))$.
Since the immediate regret of each step is at most $\Delta$, 
we immediately have $\mathrm{R}_T(p_1(\theta_1)) \le \Delta T$.
%
Furthermore, we have the following useful and intuitive monotone property, which can be proved by a dynamic-programming argument inspired by previous work~\cite[Section~3]{GM14}.
\begin{lemma} \label{lem:decreasing_regret}
$\mathrm{R}_T$ is a decreasing function of $p_1(\theta_1)$.
\end{lemma}

The proofs of the upper bounds rely on several propositions that reveal interesting recursions of Thompson Sampling's regret as a function of the prior.  Although these propositions use similar analytic techniques, they differ in many important details.  Due to space limitation, we only sketch the proof of Proposition~\ref{pro:anchor}.

\begin{proposition} \label{pro:anchor}
Consider the \theModel\ case and assume that Assumption~\ref{ass:smooth} holds. Then for any $T>0$ and $p_1(\theta_1) \in (0,1)$, we have 
$$
\mathrm{R}_T(p_1(\theta_1)) \leq  \left( 96 \log \frac{3s}{2} + 6 \right) \sqrt{\frac{T}{p_1(\theta_1)}} + \mathrm{R}_T \left( \frac{1}{3} \right)\,.
$$
\end{proposition}

\begin{proof}[Sketch]
We recall that $\theta_1$ is assumed to be the true reward-generating model in the proposition, and use the same notation as in \lemref{lem:stop}. First, the desired inequality is trivial if $p_1(\theta_1) \geq \frac{1}{3}$ since $\mathrm{R}_T(\cdot)$ is a decreasing function.
Moreover, if $\Delta \leq 2\sqrt{\frac{1}{p_1(\theta_1)T}}$, then $\mathrm{R}_T(p_1(\theta_1))\leq \Delta T \leq 2\sqrt{\frac{T}{p_1(\theta_1)}}$, which completes the proof. Thus, we can assume that $p_1(\theta_1) \leq \frac{1}{3}$ and $\Delta > 2\sqrt{\frac{1}{p_1(\theta_1)T}}$. Let $A=\frac{3}{2} p_1(\theta_1)$ and $B=\frac{1}{\Delta} \sqrt{\frac{p_1(\theta_1)}{T}}$. Then, it is easy to see that $B \leq \frac{1}{2}p_1(\theta_1) \leq \frac{1}{2} \leq 1-A$.  

Now, the first step is to upper bound $\E^{\theta_1}[\tau_A \wedge \tau_B -1]$. By Lemma~\ref{lem:tech-2b2}(a), we have for $t \leq \tau_A \wedge \tau_B -1$ that,
\begin{eqnarray*}
\E_t^{\theta_1} \left[ \log(p_t(\theta_1)^{-1}) - \log(p_{t+1}(\theta_1)^{-1}) \right]& \geq &   \frac{1}{2} p_t(\theta_1)p_t(\theta_2)^2  \Delta_1^2 + \frac{1}{2} p_t(\theta_2)^3  \Delta_2^2   \\
&\geq & \frac{p_t(\theta_2)^2 B}{2}(\Delta_1^2 + \Delta_2^2)  \, \geq  \frac{B\Delta^2}{16}  \,.
\end{eqnarray*}
In other words, $\left( \log(p_t(\theta_1)^{-1}) + t\frac{B\Delta^2}{16} \right)_{t \leq \tau_A \wedge \tau_B}$ is a supermartingale. Applying Doob's optional stopping theorem to the stopping times $\sigma_1=t \wedge \tau_A \wedge \tau_B$ and $\sigma_2=1$ and letting $t \rightarrow +\infty$ by using Lebesgue's dominated convergence theorem and the monotone convergence theorem, we have
\begin{eqnarray*}
\lefteqn{\E^{\theta_1}[\tau_A \wedge \tau_B -1] \leq  \frac{16}{B\Delta^2} \E^{\theta_1} \left[ \log \frac{p_{\tau_A \wedge \tau_B }(\theta_1)}{p_1(\theta_1)} \right]} \\
&\leq& \frac{16}{B\Delta^2} \log \frac{sA}{p_1(\theta_1)} 
=  \frac{16}{B\Delta^2}  \log \frac{3s}{2}\,,
\end{eqnarray*}
where we have used Lemma~\ref{lem:smooth} in the second last step. 

Next, the regret of Thompson Sampling can be decomposed as follows
\begin{align*}
&\lefteqn{\mathrm{R} _T(p_1(\theta_1))} \\
=&  \Delta \cdot \E^{\theta_1}  [\tau_A \wedge \tau_B -1]
         + q_{B,A} \cdot  \E^{\theta_1} [  \mathrm{R}_T(p_{\tau_B}(\theta_1)) | \tau_A >\tau_B ] \\
&         + q_{A,B} \cdot  \E^{\theta_1} [  \mathrm{R}_T(p_{\tau_A}(\theta_1)) | \tau_A <\tau_B ]   \\
   \leq&  \frac{16}{B\Delta}  \log \frac{3s}{2}
        + \frac{B}{p_1(\theta_1)} \Delta T  +  \mathrm{R}_T \left( \frac{3}{2} p_1(\theta_1) \right) \\
=&  \left( 16 \log \frac{3s}{2} + 1 \right) \sqrt{\frac{T}{p_1(\theta_1)}}
          + \mathrm{R}_T \left( \frac{3}{2} p_1(\theta_1) \right)\,,  
\end{align*}
where in the second last step, we have used the facts that $q_{B,A} \leq \frac{B}{p_1(\theta_1)}$ (by Lemma~\ref{lem:stop}), $p_{\tau_A}(\theta_1) \geq A =  \frac{3}{2} p_1(\theta_1) $, and $\mathrm{R}_T(\cdot)$ is a decreasing function (by Lemma~\ref{lem:decreasing_regret}). Because the above recurrence inequality holds for all $p_1(\theta_1) \leq \frac{1}{3}$, simple calculations lead to the desired inequality. \qed
\end{proof}

Using similar proof techniques, one can prove the following recursion:

\begin{proposition} \label{pro:anchor2}
Consider the \theModel\ case and assume that Assumption~\ref{ass:smooth} holds. Then, for any $T>0$ and $p_1(\theta_1) \leq \frac{1}{2}$, we have 
$$ \mathrm{R}_T(p_1(\theta_1)) \leq   \left(  \frac{16s}{p_1(\theta_1)^2} +1 \right) \sqrt{T}
          + \frac{1}{2}  \mathrm{R}_T\left(\frac{1}{2s} p_1(\theta_1)\right).  $$
\end{proposition}

With the technical lemmas and propositions developed so far, we are now ready to prove the first upper bound of \thmref{thm:ub}, for $p$ small.  The second bound for large $p$ can be proved in a similar fashion, although the details are quite different~\citep{Liu15Prior}.

\begin{proof}
[of the first part in \thmref{thm:ub}]
For convenience, define $\beta = 96 \log \frac{3s}{2} + 6$. By Propositions~\ref{pro:anchor} and \ref{pro:anchor2}, 
\begin{eqnarray*}
\mathrm{R}_T\left(\frac{1}{3}\right) &\leq& ( 144s +1 )\sqrt{T} + \frac{1}{2} \mathrm{R}_T\left(\frac{1}{6s} \right) \\
&\leq& ( 144s +1 )\sqrt{T}  +  \frac{1}{2} \beta \sqrt{6sT} + \frac{1}{2} \mathrm{R}_T \left( \frac{1}{3} \right)\,.
\end{eqnarray*}
Therefore, 
$$\mathrm{R}_T\left(\frac{1}{3}\right)  \leq  \left( 288s + \beta \sqrt{6s} + 2 \right) \sqrt{T}  .$$
Using again Proposition~\ref{pro:anchor}, one has for any $p_1(\theta_1) \in (0,1)$,
\begin{align*} 
\mathrm{R}_T(p_1(\theta_1)) &\leq \beta \sqrt{\frac{T}{p_1(\theta_1)}} + \mathrm{R}_T\left(\frac{1}{3}\right) \\
   &\leq \beta \sqrt{\frac{T}{p_1(\theta_1)}} +  \left( 288s + \beta \sqrt{6s} + 2 \right) \sqrt{T}   \\
  &\leq \beta \sqrt{\frac{T}{p_1(\theta_1)}} +  \left( 288s + \beta \sqrt{6s} + 2 \right) \sqrt{\frac{T}{p_1(\theta_1)}}  \\
  &\leq \left( 288s + \beta (\sqrt{6s}+1) + 2 \right) \sqrt{\frac{T}{p_1(\theta_1)}} 
   \,\,\leq\,\, 1490s \sqrt{\frac{T}{p_1(\theta_1)}}\,,
\end{align*}
where the last step follows from the inequalities $\beta=  96 \log \frac{3s}{2} + 6 \leq 300 \sqrt{s}$ and $\sqrt{6s}+1 \leq 4\sqrt{s}$ for $s > 1$. \qed
\end{proof}

\section{Lower Bounds}
\label{sec:lb}

In this section, we give a proof for the lower bound when the prior is poor (\thmref{thm:lb-poor}); the other case (\thmref{thm:lb-good}) is left in the long version~\citep{Liu15Prior}. The following technical lemma is needed, which can be proved by direct calculations:

\begin{lemma} \label{lem:lr-ub}
Let $ -\sqrt{\frac{1}{8}} \leq \Delta \leq \sqrt{\frac{1}{8}}$. Let $\ell_1$ and $\ell_2$ be the density functions of the Bernoulli distributions $Bern\left(\frac{1}{2}+ \Delta \right)$ and $Bern\left(\frac{1}{2}-\Delta\right)$ with respect to the counting measure on $[0,1]$. Then $\E_{X \sim Bern\left(\frac{1}{2}+ \Delta \right)} \left[\frac{\ell_1(X)}{\ell_2(X)}-1 \right] \leq 32 \Delta^2 $.
\end{lemma}

\begin{proof}
[of \thmref{thm:lb-poor}]
Let $A=\frac{3}{2}p_1(\theta_1)$.  Clearly, $A \le \frac{3}{4}$.  Recall that $\tau_{A} = \inf \{t \geq 1, p_t(\theta_1) \geq A \}$. Using Lemma~\ref{lem:tech-2b2}(b) and Lemma~\ref{lem:lr-ub}, one has for $t \leq \tau_{A}  -1$,
\begin{eqnarray*}
  \lefteqn{\E_t^{\theta_1} \left[ p_{t+1}(\theta_1) - p_t(\theta_1)  \right]} \\
  &\leq& \sum_{i \in \{1,2\} } p_t(\theta_i)  p_{t}(\theta_1)  p_{t}(\theta_2)     \E^{\theta_1}\left[  \frac{\ell_{i}(\theta_1)(X_{i,t}) }{\ell_{i}(\theta_2)(X_{i,t}) }  - 1 \right] \\
  & = & p_{t}(\theta_1)^2  p_{t}(\theta_2)     \E^{\theta_1}\left[  \frac{\ell_{1}(\theta_1)(X_{1,t}) }{\ell_{1}(\theta_2)(X_{1,t}) }  - 1 \right] 
  \leq 32 A^2 \Delta^2 = 72 p_1(\theta_1)^2 \Delta^2 .
\end{eqnarray*} 
Therefore, $\left( p_t(\theta_1) - 72p_1(\theta_1)^2 \Delta^2 t \right)_{t \leq \tau_{A}}$ is a supermartingale. Now, using Doob's optional stopping theorem, one has  
$\E^{\theta_1} \left[ p_{t \wedge \tau_{A} \wedge T  }(\theta_1) - (t \wedge \tau_{A} \wedge T  )72 p_1(\theta_1)^2 \Delta^2 \right]   \leq p_1(\theta_1) - 72 p_1(\theta_1)^2 \Delta^2$ for any $t \geq 1$.

Moreover, using Lebesgue's dominated convergence theorem and the monotone convergence theorem, 
\begin{align*}
& \E^{\theta_1} \left[ p_{t \wedge \tau_{A} \wedge T   }(\theta_1) - (t \wedge \tau_{A} \wedge T  )72 p_1(\theta_1)^2 \Delta^2 \right] \\
& \longrightarrow \E^{\theta_1} \left[ p_{ \tau_{A} \wedge T  }(\theta_1) - ( \tau_{A} \wedge T )72 p_1(\theta_1)^2 \Delta^2 \right]
\end{align*}
as $t \rightarrow +\infty$. Hence,
$$ \E^{\theta_1}[ \tau_{A} \wedge T -1] \geq  \frac{1}{72 p_1(\theta_1)^2 \Delta^2 }  \E^{\theta_1} \left[ p_{ \tau_{A} \wedge T  }(\theta_1) -  p_1(\theta_1) \right] .  $$
One one side, if $\P^{\theta_1}(\tau_A \wedge T = T) \geq \frac{1}{21}$, then $\E^{\theta_1}[\tau_A \wedge T ] \geq \P^{\theta_1}(\tau_A \wedge T = T) T \geq \frac{T}{21}$. On the other side, if $\P^{\theta_1}(\tau_A \wedge T = \tau_A ) \geq \frac{20}{21}$, then $\E^{\theta_1} \left[ p_{ \tau_{A} \wedge T  }(\theta_1)  \right]  \geq \P^{\theta_1}(\tau_A \wedge T = \tau_A) A \geq \frac{10}{7} p_1(\theta_1)$ and thus $$ \E^{\theta_1}[ \tau_{A} \wedge T -1] \geq  \frac{1}{72 p_1(\theta_1)^2 \Delta^2 }  \left( \frac{10}{7} p_1(\theta_1) -  p_1(\theta_1) \right) = \frac{T}{21} .  $$ 
In both cases, we have $\E^{\theta_1}[\tau_A \wedge T -1 ]  \geq \frac{T}{21}$.

Finally, one has 
\begin{eqnarray*}
\mathrm{R}_T(\theta_1, TS(p_1)) &=& \Delta \E^{\theta_1} \left[ \sum_{t=1}^{T} (1-p_t(\theta_1)) \right] \geq \Delta \E^{\theta_1} \left[ \sum_{t=1}^{ \tau_A \wedge T -1 } (1-p_t(\theta_1)) \right]  \\
& \geq &  \Delta (1-A) \E^{\theta_1}[\tau_A \wedge T -1 ] 
\geq \frac{\Delta T}{84}  = \frac{1}{168\sqrt{2}} \sqrt{ \frac{T}{p_1(\theta_1)}}\,,
\end{eqnarray*}
where we have used the fact that $1-A \geq \frac{1}{4}$. \qed
\end{proof}

\begin{proof}
[of \thmref{thm:lb-good}]
Using Lemma~\ref{lem:tech-2b2}(c) and Lemma~\ref{lem:lr-ub}, one has
\begin{eqnarray*}
 \E_t^{\theta_1} \left[ p_{t+1}(\theta_2)^{-1} - p_t(\theta_2)^{-1}  \right] 
  & = &  \sum_{i \in \{1,2\} } p_t(\theta_i) \frac{p_{t}(\theta_1)}{p_{t}(\theta_2)}  \E^{\theta_1}\left[  \frac{  \ell_{i}(\theta_1)(X_{i,t}) }{  \ell_{i}(\theta_2)(X_{i,t})  } - 1  \right] \\   
  & = &  p_{t}(\theta_1) \E^{\theta_1}\left[  \frac{  \ell_{2}(\theta_1)(X_{2,t}) }{  \ell_{2}(\theta_2)(X_{2,t})  } - 1  \right]  
        \, \leq \, 32 \Delta^2 . 
\end{eqnarray*} 
Then for any $t \leq T$,
\[
\E^{\theta_1} \left[ p_{t}(\theta_2)^{-1}  \right] \leq \frac{1}{1-p_1(\theta_1)} + 32(t-1)\Delta^2  = \frac{1+{4(t-1)}/{T}}{1-p_1(\theta_1)} \leq \frac{5}{1-p_1(\theta_1)} \,.
\]
By Jensen's inequality, we have for any $t \leq T$, 
$\E^{\theta_1} \left[ p_{t}(\theta_2) \right] \geq  \left( \E^{\theta_1} \left[ p_{t}(\theta_2)^{-1}  \right] \right)^{-1}   \geq \frac{1-p_1(\theta_1)}{5}$.
Hence,
\[
\mathrm{R}_T(\theta_1, TS(p_1)) = \Delta \cdot \E^{\theta_1}  \sum_{t=1}^T  p_t(\theta_2) \geq \Delta T \frac{1-p_1(\theta_1)}{5} \geq \frac{1}{10\sqrt{2}} \sqrt{(1- p_1(\theta_1))T} \,.
\]
\qed
\end{proof}

\section{Conclusions} \label{sec:conclusions}

In this work, we studied an important aspect of the popular Thompson Sampling strategy for stochastic bandits --- its sensitivity to the prior.  Focusing on a special yet nontrivial problem, we fully characterized its worst-case dependence of regret on prior, both for the good- and bad-prior cases, with matching upper and lower bounds.  The lower bounds are also extended to a more general case as a corollary, quantifying inherent sensitivity of the algorithm when the prior is poor and when no structural assumptions are made.

These results suggest a few interesting directions for future work, only four of which are outlined here.  One is to close the gap between upper and lower bounds for the general, multiple-model case.  We conjecture that a tighter upper bound is likely to match the lower bound in Corollary~\ref{cor:general_lb}.  The second is to consider prior sensitivity for structured stochastic bandits, where models in $\Theta$ are related in certain ways.  For example, in the discretized version of the multi-armed bandit problem~\citep{AG13}, the prior probability mass of the true model is exponentially small when a uniform prior is used, but strong frequentist regret bound is still possible.  Sensitivity analysis for such problems can provide useful insights and guidance for applications of Thompson Sampling. Thrid, it remains open whether there exists an algorithm whose worst-case regret bounds are better than those of Thompson Sampling for any range of $p_1(\theta^{*})$, with $\theta^*$ being the true underlying model.  This question is related to the recent study of Pareto regret front~\citep{Lattimore15Pareto}.  We conjecture that the answer is negative, especially in the \theModel\ case.  Finally, it is interesting to consider problem-dependent regret bounds that often scale logarithmically with $T$.

\subsubsection*{Acknowledgments}

We thank S\'{e}bastien Bubeck and the anonymous reviewers for helpful advice that improves the presentation of the paper.

%


\bibliographystyle{plain}
\bibliography{newbib_brief}


\clearpage
\appendix

\begin{center} \textbf{\large{Appendix to \\ \textit{On the Prior Sensitivity of Thompson Sampling } }} \end{center}

\section{Technical Lemmas}\label{sec:tech_lems}

\subsection{Proof of \lemref{lem:smooth}}

\begin{proof}[of \lemref{lem:smooth}]
Without loss of generality, we assume that $\theta = \theta_1$. Recall that
\begin{eqnarray*}
 \frac{p_{t+1}(\theta_1)}{p_{t}(\theta_1)} &=&  \frac{   \ell_{I_t}(\theta_1)(X_{I_t,t}) } { p_{t}(\theta_1)  \ell_{I_t}(\theta_1)(X_{I_t,t}) + p_{t}(\theta_2) \ell_{I_t}(\theta_2)(X_{I_t,t})   } \\
    &=&  \frac{1}{ p_{t}(\theta_1)   + p_{t}(\theta_2) \frac{\ell_{I_t}(\theta_2)(X_{I_t,t})}{\ell_{I_t}(\theta_1)(X_{I_t,t})}   } 
\end{eqnarray*} 
Therefore, we have 
$$ \frac{1}{s} \leq  \frac{1}{ p_{t}(\theta_1)   + p_{t}(\theta_2) s   }   \leq \frac{p_{t+1}(\theta_1)}{p_{t}(\theta_1)}   \leq  \frac{1}{ p_{t}(\theta_1)   + p_{t}(\theta_2) \frac{1}{s}   }  \leq s , $$
which completes the proof. \qed
\end{proof}

\subsection{Proof of \lemref{lem:tech-2b2}}

The proof of \lemref{lem:tech-2b2}, key to the upper-bound analysis, relies on the following result:

\begin{lemma} \label{lem:kl-lb}
Let $\alpha \in [0,1]$. Let $\nu_1$ and $\nu_2$ be two probability distributions on $[0,1]$ with mean $\mu_1$ and $\mu_2$, Then we have
$$KL(\nu_1, \alpha\nu_1+(1-\alpha)\nu_2) \geq \frac{(1-\alpha)^2}{2}|\mu_1-\mu_2|^2 . $$  
\end{lemma}
\begin{proof} 
Let $\nu_1$ and $\nu_2$ be absolutely continuous with respect to some measure $v$ with density functions $\ell_1$ and $\ell_2$. On one hand, by Pinsker's inequality, we have 
$$ KL(\nu_1, \alpha\nu_1+(1-\alpha)\nu_2) \geq \frac{1}{2}  \left( \int_0^1 |\ell_1(x)-\alpha\ell_1(x)-(1-\alpha)\ell_2(x) | \,  \mathrm{d} v(x) \right)^2 .  $$
On the other hand, 
$$|\mu_1-\mu_2| = \left| \int_0^1 ( \ell_1(x)x - \ell_2(x)x )  \,  \mathrm{d} v(x) \right|
   \leq \int_0^1 | \ell_1(x) - \ell_2(x)|  \,  \mathrm{d} v(x)   $$
which completes the proof. \qed
\end{proof}

We are now ready to prove \lemref{lem:tech-2b2}.

\begin{proof}[of \lemref{lem:tech-2b2}] 
Recall that for the \theModel, 
$$  p_{t+1}(\theta_1)  =  \frac{ p_{t}(\theta_1)  \ell_{I_t}(\theta_1)(X_{I_t,t}) } { p_{t}(\theta_1)  \ell_{I_t}(\theta_1)(X_{I_t,t})  +  p_{t}(\theta_2)  \ell_{I_t}(\theta_2)(X_{I_t,t})  } , $$
$$  p_{t+1}(\theta_2)  =  \frac{ p_{t}(\theta_2)  \ell_{I_t}(\theta_2)(X_{I_t,t})  } { p_{t}(\theta_1)  \ell_{I_t}(\theta_1)(X_{I_t,t})  +  p_{t}(\theta_2)  \ell_{I_t}(\theta_2)(X_{I_t,t})  }  $$
and $I_t=i$ with probability $p_t(\theta_i)$ for $i \in \{1,2\}$.
We carry out the following computations to prove the lemma. 

\noindent\textbf{(a)} 
\begin{eqnarray*}
\lefteqn{\E_t^{\theta_1} \left[ \log(p_t(\theta_1)^{-1}) - \log(p_{t+1}(\theta_1)^{-1}) \right]} \\
  & = &  \E_t^{\theta_1}\left[ \log \frac{  \ell_{I_t}(\theta_1)(X_{I_t,t}) }{ p_{t}(\theta_1)  \ell_{I_t}(\theta_1)(X_{I_t,t})  +  p_{t}(\theta_2)  \ell_{I_t}(\theta_2)(X_{I_t,t})  } \right] \\
 & = & \sum_{i \in \{1,2\} } p_t(\theta_i)  \E_t^{\theta_1}\left[ \log \frac{  \ell_{i}(\theta_1)(X_{i,t}) }{ p_{t}(\theta_1)  \ell_{i}(\theta_1)(X_{i,t})  +  p_{t}(\theta_2)  \ell_{i}(\theta_2)(X_{i,t})  } \right] \\
 & = & \sum_{i \in \{1,2\} } p_t(\theta_i) KL(\nu_i(\theta_1), p_t(\theta_1)\nu_i(\theta_1)+p_t(\theta_2)\nu_i(\theta_2) )\\
 & \geq & \sum_{i \in \{1,2\} } \frac{1}{2} p_t(\theta_i)p_t(\theta_2)^2  |\mu_i(\theta_1)-\mu_i(\theta_2)|^2   \,,
\end{eqnarray*}
where the last step follows from Lemma~\ref{lem:kl-lb}.

\noindent\textbf{(b)} 
\begin{eqnarray*}
  \lefteqn{\E_t^{\theta_1} \left[ (1-p_{t+1}(\theta_1))^{-1} - (1-p_t(\theta_1))^{-1}  \right]} \\
  & =  &  \E_t^{\theta_1} \left[ p_{t+1}(\theta_2)^{-1} - p_t(\theta_2)^{-1}  \right] \\
  & = &  \E_t^{\theta_1}\left[  \frac{ p_{t}(\theta_1)  \ell_{I_t}(\theta_1)(X_{I_t,t})  +  p_{t}(\theta_2)  \ell_{I_t}(\theta_2)(X_{I_t,t})  }{ p_{t}(\theta_2)  \ell_{I_t}(\theta_2)(X_{I_t,t})  } - \frac{1}{p_t(\theta_2)}    \right] \\
  & = & \frac{p_{t}(\theta_1)}{p_{t}(\theta_2)}  \E_t^{\theta_1}\left[  \frac{  \ell_{I_t}(\theta_1)(X_{I_t,t}) }{  \ell_{I_t}(\theta_2)(X_{I_t,t})  } - 1  \right] \\
  & = &  \sum_{i \in \{1,2\} } p_t(\theta_i) \frac{p_{t}(\theta_1)}{p_{t}(\theta_2)}  \E^{\theta_1}\left[  \frac{  \ell_{i}(\theta_1)(X_{i,t}) }{  \ell_{i}(\theta_2)(X_{i,t})  } - 1  \right] \\ 
  & \geq & \sum_{i \in \{1,2\} } p_t(\theta_i) \frac{p_{t}(\theta_1)}{p_{t}(\theta_2)}  \E^{\theta_1}\left[ \log \frac{  \ell_{i}(\theta_1)(X_{i,t}) }{  \ell_{i}(\theta_2)(X_{i,t})  }  \right] \\ 
 & = & \frac{p_{t}(\theta_1)^2}{p_{t}(\theta_2)}  KL( \nu_{1}(\theta_1), \nu_{1}(\theta_2) )  +   p_t(\theta_1)   KL(\nu_{2}(\theta_1), \nu_{2}(\theta_2) )  \\
 & = & \frac{p_{t}(\theta_1)^2}{2p_{t}(\theta_2)} |\mu_1(\theta_1)-\mu_1(\theta_2)|^2+   \frac{p_t(\theta_1)}{2} |\mu_2(\theta_1)-\mu_2(\theta_2)|^2
\end{eqnarray*}  
where we have used the inequality $x-1 \geq \log x$  and the last step follows from \lemref{lem:kl-lb}.

\noindent\textbf{(c)} 
\begin{eqnarray*}
  \lefteqn{\E_t^{\theta_1} \left[ p_{t+1}(\theta_1) - p_t(\theta_1)  \right]} \\
  & =  &   p_{t}(\theta_1)    \E_t^{\theta_1}\left[ \frac{   \ell_{I_t}(\theta_1)(X_{I_t,t}) } { p_{t}(\theta_1)  \ell_{I_t}(\theta_1)(X_{I_t,t})  +  p_{t}(\theta_2)  \ell_{I_t}(\theta_2)(X_{I_t,t})  }  - 1 \right] \\
  & = &  \sum_{i \in \{1,2\} } p_t(\theta_i)  p_{t}(\theta_1)    \E_t^{\theta_1}\left[ \frac{ \ell_{i}(\theta_1)(X_{i,t}) } { p_{t}(\theta_1)  \ell_{i}(\theta_1)(X_{i,t})  +  p_{t}(\theta_2)  \ell_{i}(\theta_2)(X_{i,t})  }  - 1 \right] .
\end{eqnarray*}
On one hand, using the inequality $x-1 \geq \log x$, we have 
\begin{eqnarray*}
  &  &  \E_t^{\theta_1} \left[ p_{t+1}(\theta_1) - p_t(\theta_1)  \right] \\
  & \geq &  \sum_{i \in \{1,2\} } p_t(\theta_i)  p_{t}(\theta_1)    \E_t^{\theta_1}\left[ \log \frac{ \ell_{i}(\theta_1)(X_{i,t}) } { p_{t}(\theta_1)  \ell_{i}(\theta_1)(X_{i,t})  +  p_{t}(\theta_2)  \ell_{i}(\theta_2)(X_{i,t})  }   \right] \\
  & = &  \sum_{i \in \{1,2\} } p_t(\theta_i)  p_{t}(\theta_1)    KL \left( \nu_{i}(\theta_1) , p_{t}(\theta_1)  \nu_{i}(\theta_1) +  p_{t}(\theta_2)  \nu_{i}(\theta_2) \right)   \, \geq \, 0.
\end{eqnarray*} 
On the other hand, using Jensen's inequality on the convex function $x \rightarrow x^{-1}$, one has
\begin{eqnarray*}
  &  &  \E_t^{\theta_1} \left[ p_{t+1}(\theta_1) - p_t(\theta_1)  \right] \\
  & \leq  & \sum_{i \in \{1,2\} } p_t(\theta_i)  p_{t}(\theta_1)    \E_t^{\theta_1}\left[  \ell_{i}(\theta_1)(X_{i,t}) \left( \frac{ p_{t}(\theta_1)}{\ell_{i}(\theta_1)(X_{i,t})}  + \frac{ p_{t}(\theta_2)}{\ell_{i}(\theta_2)(X_{i,t}) } \right)  - 1 \right] \\
  & =  &   \sum_{i \in \{1,2\} } p_t(\theta_i)  p_{t}(\theta_1)  p_{t}(\theta_2)     \E^{\theta_1}\left[  \frac{\ell_{i}(\theta_1)(X_{i,t}) }{\ell_{i}(\theta_2)(X_{i,t}) }  - 1 \right]. 
\end{eqnarray*}

\noindent\textbf{(d)}
By definition of the regret and part(c), one has
\[
\mathrm{R}_T(\theta_1, TS(p_1)) = \Delta  \E^{\theta_1}  \sum_{t=1}^T  p_t(\theta_2) = \Delta  \E^{\theta_1}  \sum_{t=1}^T (1- p_t(\theta_1)) \leq \Delta T (1- p_1(\theta_1)) \,.
\]
\qed
\end{proof}

\subsection{Proof of \lemref{lem:decreasing_regret}}

\begin{proof}[of \lemref{lem:decreasing_regret}]
The proof is inspired by the dynamic-programming argument used in Section~3 of a previous study~\citep{GM14}. We assume that $\theta_1$ is the true reward-generating model. For arm $i\in\{1,2\}$, define $R_T^{(i)}(\alpha)$ as the regret of the policy that starts with the prior $p_1=(\alpha,1-\alpha)$, plays arm $i$ for the first step, and then executes Thompson Sampling for the remaining $T-1$ steps.  It is easy to see that
\begin{equation}
R_T(\alpha) = \alpha R_T^{(1)}(\alpha) + (1-\alpha)R_T^{(2)}(\alpha)\,. \label{eqn:dp-recurrence}
\end{equation}
We now prove by induction on $T$ that $R_T(\cdot)$ is a decreasing function.  For the base case of $T=1$, $R(\alpha)=1-\alpha$ is obviously decreasing.  Now, suppose $R_t(\cdot)$ is decreasing for all $t<T$, and we will show that $R_T(\cdot)$ is also decreasing.  The proof proceeds in three main steps.

\noindent\textbf{Step One:} This step is devoted to showing that both $R_T^{(1)}$ and $R_T^{(2)}$ are decreasing functions of $\alpha$.  By definition, we have
\begin{eqnarray*}
R_T^{(1)}(\alpha) &=& \E_{Z\sim\mu_1(\theta_1)}\left[R_{T-1}\left(\frac{\alpha\ell_1(\theta_1)(Z)}{\alpha\ell_1(\theta_1)(Z)+(1-\alpha)\ell_1(\theta_2)(Z)}\right)\right] \\
R_T^{(2)}(\alpha) &=& \Delta + \E_{Z\sim\mu_2(\theta_1)}\left[R_{T-1}\left(\frac{\alpha\ell_2(\theta_1)(Z)}{\alpha\ell_2(\theta_1)(Z)+(1-\alpha)\ell_2(\theta_2)(Z)}\right)\right]\,.
\end{eqnarray*}
Since $R_{T-1}(\alpha)$ is decreasing with $\alpha\in(0,1)$, it follows that
\[
R_T^{(1)}(\alpha) = \E_{Z\sim\mu_1(\theta_1)} \left[ R_{T-1}\left(\frac{\ell_1(\theta_1)(Z)}{\ell_1(\theta_1)(Z)+(1/\alpha-1)\ell_1(\theta_2)(Z)}\right) \right]
\]
is a decreasing function of $\alpha$.  Similarly, $R_T^{(2)}(\alpha)$ is also a decreasing function.

\noindent\textbf{Step Two:} This step is to show that the functions $R_T^{(1)}$ and $R_T^{(2)}$ satisfy
\begin{equation}
R_T^{(1)}(\alpha) \le R_T^{(2)}(\alpha) \label{eqn:R1-lt-R2}
\end{equation}
for any $T$ and $\alpha\in(0,1)$.  We prove the claim by mathematical induction on $T$.  The base case where $T=1$ is trivial, since $R_1^{(1)}(\alpha) \equiv 0$ and $R_1^{(2)}(\alpha) \equiv \Delta$.  Now suppose $R_t^{(1)}(\alpha) \le R_t^{(2)}(\alpha)$ for all $t<T$.  Then for every $t<T$, $R_t^{(1)}(\alpha) \le R_t(\alpha) \le R_t^{(2)}(\alpha)$, because of Equation~\ref{eqn:dp-recurrence} and the induction hypothesis.  It follows that,
\begin{eqnarray*}
R_T^{(1)}(\alpha) &=& \E_{Z\sim\mu_1(\theta_1)}\left[R_{T-1}\left(\frac{\alpha\ell_1(\theta_1)(Z)}{\alpha\ell_1(\theta_1)(Z)+(1-\alpha)\ell_1(\theta_2)(Z)}\right)\right] \\
  &\le& \E_{Z\sim\mu_1(\theta_1)}\left[R_{T-1}^{(2)}\left(\frac{\alpha\ell_1(\theta_1)(Z)}{\alpha\ell_1(\theta_1)(Z)+(1-\alpha)\ell_1(\theta_2)(Z)}\right)\right] \\
  &=& \Delta + \E_{Z\sim\mu_1(\theta_1)}\left[\E_{Z'\sim\mu_2(\theta_1)}\left[X^{(1)}\right]\right]\,,
\end{eqnarray*}
where
\[
X^{(1)} = R_{T-2}\left(\frac{\alpha\ell_1(\theta_1)(Z)\ell_2(\theta_1)(Z')}{\alpha\ell_1(\theta_1)(Z)\ell_2(\theta_1)(Z')+(1-\alpha)\ell_1(\theta_2)(Z)\ell_2(\theta_2)(Z')}\right)\,;
\]
and that
\begin{eqnarray*}
R_T^{(2)}(\alpha) &=& \Delta + \E_{Z\sim\mu_2(\theta_1)}\left[R_{T-1}\left(\frac{\alpha\ell_2(\theta_1)(Z)}{\alpha\ell_2(\theta_1)(Z)+(1-\alpha)\ell_2(\theta_2)(Z)}\right)\right] \\
  &\ge& \Delta + \E_{Z\sim\mu_2(\theta_1)}\left[R_{T-1}^{(1)}\left(\frac{\alpha\ell_2(\theta_1)(Z)}{\alpha\ell_2(\theta_1)(Z)+(1-\alpha)\ell_2(\theta_2)(Z)}\right)\right] \\
  &=& \Delta + \E_{Z\sim\mu_2(\theta_1)}\left[\E_{Z'\sim\mu_1(\theta_1)}\left[X^{(2)}\right]\right]\,,
\end{eqnarray*}
where
\[
X^{(2)} = R_{T-2}\left(\frac{\alpha\ell_1(\theta_1)(Z')\ell_2(\theta_1)(Z)}{\alpha\ell_1(\theta_1)(Z')\ell_2(\theta_1)(Z) + (1-\alpha)\ell_1(\theta_2)(Z')\ell_2(\theta_2)(Z)}\right)\,.
\]
Thus, $R_T^{(2)}(\alpha) \ge R_T^{(1)}(\alpha)$ by Fubini's theorem.

\noindent\textbf{Step Three:} This step finishes the induction step, based on results established in the previous two steps.  For any $0<\alpha<\beta<1$, we have
\begin{eqnarray*}
R_T(\beta) &=& \beta R_T^{(1)}(\beta) + (1-\beta)R_T^{(2)}(\beta) \\
&\le& \beta R_T^{(1)}(\alpha) + (1-\beta)R_T^{(2)}(\alpha) \\
&\le& \alpha R_T^{(1)}(\alpha) + (1-\alpha)R_T^{(2)}(\alpha) \\
&=& R_T(\alpha)\,,
\end{eqnarray*}
where the equalities are from Equation~\ref{eqn:dp-recurrence}, the first inequality is from the monotonicity of $R_T^{(i)}(\cdot)$ established in Step One, and the second is from \eqnref{eqn:R1-lt-R2}.  We have thus proved that $R_t(\cdot)$ is a decreasing function for $t=T$, and finished the inductive step. \qed
\end{proof}

\subsection{Markov Property}

Another fundamental, although intuitive, property of Thompson Sampling is that the posterior distribution it maintains over the set of models forms a Markov process.  This property is used in the proofs of multiple propositions in later sections.

\begin{lemma}{(Markov Property)} \label{lem:markov}
Regardless of the true underlying model, the stochastic process $( p_{t})_{t \geq 1}$ is a Markov process.
\end{lemma}

\begin{proof}[of \lemref{lem:markov}]
Let $\theta^*$ be the true underlying model. Recall that $$ p_{t+1}(\theta) = \frac{ p_{t}(\theta)  \ell_{I_t}(\theta)(X_{I_t,t}) } { \sum_{\eta \in \Theta} p_{t}(\eta)  \ell_{I_t}(\eta)(X_{I_t,t})   } . $$
Note that $I_t$ is drawn from $p_t$ independent of the past and $X_{i,t}$ is drawn from $\mu_i(\theta^*)$. Hence, the distribution of $p_{t+1}$ only depends on $p_t$ and $\mu_i(\theta), i=1,\hdots, K, \theta \in \Theta$. The reward distributions  $\mu_i(\theta)$ are fixed before the evolution of the process $p_t$. Thus, the distribution of $p_{t+1}$ only depends on $p_t$, not on $p_s, s=1, \hdots, t-1$. This shows that $p_t$ is a Markov process. \qed
\end{proof}

\section{Proof of \corref{cor:general_lb}}   \label{sec:general_lb-proof}

\begin{proof}[of \corref{cor:general_lb}]
Let $\tilde{p}_1$ be the prior over $\{\theta_1, \theta_2 \}$ defined as $\tilde{p}_1(\theta_1)=p_1(\theta^*)$ and  $\tilde{p}_1(\theta_2)=p_1(\Theta \backslash \{ \theta^* \})$. By \thmref{thm:lb-poor}, there exists a \theModel\ problem instance $\mathcal{P}$ (defined by $\nu_i(\theta_j), i,j=1,2$) where the regret of TS with prior $\tilde{p}_1$ is $\Omega(\sqrt{\frac{T}{\tilde{p}_1(\theta_1)}})$ for small $\tilde{p}_1(\theta_1)$. Now consider the problem instance $\mathcal{Q}$ for the general $\Theta$ case defined as $\nu_i(\theta^*)=\nu_i(\theta_1)$ for $i=1,2$ and $\nu_i(\theta)=\nu_i(\theta_2)$ for $i=1,2$ and $\theta \in \Theta \backslash \{ \theta^* \}$. It is easy to see that Thompson Sampling with prior $p_1$ under $\mathcal{Q}$ has exactly the same regret as Thompson Sampling with prior $\tilde{p}_1$ under $\mathcal{P}$. Thus, under $\mathcal{Q}$, the regret of Thompson Sampling with prior $p_1$ is $\Omega(\sqrt{\frac{T}{\tilde{p}_1(\theta_1)}}) = \Omega(\sqrt{\frac{T}{p_1(\theta^*)}})$ for small $p_1(\theta^*)$. The $ \Omega(\sqrt{ (1- p_1(\theta^*))T })$ lower bound for large $p_1(\theta^*)$ can be similarly obtained. \qed
\end{proof}

\section{Proof of \thmref{thm:ub}} \label{sec:ub-proof}

In the main text, we only sketch the proof for the first part of the theorem.  Here, a full proof is presented, which requires an additional proposition that plays a similar role as Propositions~\ref{pro:anchor} and \ref{pro:anchor2}.  Its proof is given in \secref{sec:anchor3}.

\begin{proposition} \label{pro:anchor3}
Consider the \theModel\ case and assume that Assumption~\ref{ass:smooth} holds. We also assume that $\Delta \geq \frac{1}{\sqrt{(1-p_1(\theta_1))T}}$ and define the function $\mathrm{Q}_T(\cdot)$ by $\mathrm{Q}_T(x) = \mathrm{R}_T(1-x)$. Then for any $T>0$ and  $p_1(\theta_2) \leq \frac{1}{8s^2}$, we have
\begin{align*}
&\mathrm{Q}_T(p_1(\theta_2)) - \mathrm{Q}_T(\frac{1}{4s^2}p_1(\theta_2)) \\
&\leq 360 s^4 \sqrt{p_1(\theta_2)T}  +  \frac{4}{11s}\left(  \mathrm{Q}_T\left(4s^2 p_1(\theta_2) \right) - \mathrm{Q}_T(p_1(\theta_2)) \right)\,.
\end{align*}
\end{proposition}

\begin{proof}[of \thmref{thm:ub}]

\noindent\textbf{Proof of the First Inequality:} Let $\beta=  96 \log \frac{3s}{2} + 6$. By Propositions~\ref{pro:anchor} and \ref{pro:anchor2}, 
\begin{align*}
\mathrm{R}_T\left(\frac{1}{3}\right)  &\leq  ( 144s +1 )\sqrt{T} + \frac{1}{2} \mathrm{R}_T\left(\frac{1}{6s} \right) \\
 &\leq   ( 144s +1 )\sqrt{T}  +  \frac{1}{2} \beta \sqrt{6sT} + \frac{1}{2} \mathrm{R}_T \left( \frac{1}{3} \right)\,.
\end{align*}
Therefore, 
$$\mathrm{R}_T\left(\frac{1}{3}\right)  \leq  \left( 288s + \beta \sqrt{6s} + 2 \right) \sqrt{T}  .$$
Using again Proposition~\ref{pro:anchor}, one has for any $p_1(\theta_1) \in (0,1)$,
\begin{eqnarray*} 
\mathrm{R}_T(p_1(\theta_1)) &\leq& \beta \sqrt{\frac{T}{p_1(\theta_1)}} + \mathrm{R}_T\left(\frac{1}{3}\right)    \\
   &\leq & \beta \sqrt{\frac{T}{p_1(\theta_1)}} +  \left( 288s + \beta \sqrt{6s} + 2 \right) \sqrt{T}   \\
  &\leq & \beta \sqrt{\frac{T}{p_1(\theta_1)}} +  \left( 288s + \beta \sqrt{6s} + 2 \right) \sqrt{\frac{T}{p_1(\theta_1)}}  \\
  &\leq & \left( 288s + \beta (\sqrt{6s}+1) + 2 \right) \sqrt{\frac{T}{p_1(\theta_1)}} \\
  &\leq& 1490s \sqrt{\frac{T}{p_1(\theta_1)}}\,,
\end{eqnarray*}
where the last step follows from the inequalities $\beta=  96 \log \frac{3s}{2} + 6 \leq 300 \sqrt{s}$ and $\sqrt{6s}+1 \leq 4\sqrt{s}$ for $s > 1$.

\noindent\textbf{Proof of the Second Inequality:}
Fix $p_1(\theta_1) \geq 1-\frac{1}{8s^2}$. First, if $\Delta \leq \frac{1}{\sqrt{(1-p_1(\theta_1))T}}$, then by Lemma~\ref{lem:tech-2b2}(d), $\mathrm{R}_T(p_1(\theta_1)) \leq (1-p_1(\theta_1))\Delta T \leq \sqrt{(1-p_1(\theta_1))T}$. Hence, we can assume that $\Delta \geq \frac{1}{\sqrt{(1-p_1(\theta_1))T}}$.  It follows from \propref{pro:anchor3}  that for any integer  $h \geq 1$, as long as $(4s^2)^{h-1}p_1(\theta_2) \leq \frac{1}{8s^2}$, one has
\begin{eqnarray*}
 \lefteqn{\mathrm{Q}_T(p_1(\theta_2)) - \mathrm{Q}_T(\frac{1}{4s^2}p_1(\theta_2))} \\
 &\leq & \sum_{k=0}^{h-1}  \left( \frac{4}{11s} \right)^k 360s^4 \sqrt{(4s^2)^k p_1(\theta_2)T} + \left( \frac{4}{11s} \right)^h \mathrm{Q}_T((4s^2)^h p_1(\theta_2))  \\ 
 &\leq & \sum_{k=0}^{h-1}  \left( \frac{8}{11} \right)^k 360s^4 \sqrt{p_1(\theta_2)T} + \left( \frac{4}{11s} \right)^h \mathrm{Q}_T((4s^2)^h p_1(\theta_2))  \\ 
 &\leq &  1320 s^4 \sqrt{p_1(\theta_2)T} + \left( \frac{4}{11s} \right)^h \mathrm{Q}_T((4s^2)^h p_1(\theta_2)) \,.
\end{eqnarray*}
Let $h$ be the smallest integer such that $(4s^2)^{h}p_1(\theta_2) > \frac{1}{8s^2}$. On one hand,  $(4s^2)^{h-1}p_1(\theta_2) \leq \frac{1}{8s^2}$ implies that $1 - (4s^2)^{h}p_1(\theta_2) \geq \frac{1}{2}$. Using the first inequality of Theorem~\ref{thm:ub}  and the fact that the function $\mathrm{R}_T(\cdot) $ is decreasing, one has 
\begin{align*}
\mathrm{Q}_T((4s^2)^h p_1(\theta_2)) &= \mathrm{R}_T(1-(4s^2)^h p_1(\theta_2)) \\
&\leq \mathrm{R}_T(\frac{1}{2})
\leq 1490s \sqrt{2T}\,.
\end{align*}
On the other hand, $(4s^2)^{h}p_1(\theta_2) > \frac{1}{8s^2}$ implies that
$ 2\sqrt{2}s \sqrt{p_1(\theta_2)} > \left(\frac{1}{2s}\right)^{h} > \left( \frac{4}{11s} \right)^h. $
Hence, for $p_1(\theta_2) \leq \frac{1}{8s^2}$,
\begin{align*}
\mathrm{Q}_T(p_1(\theta_2)) - \mathrm{Q}_T(\frac{1}{4s^2}p_1(\theta_2)) &\leq (1320 s^4+5960 s^2) \sqrt{p_1(\theta_2)T} \\
&\leq 7280 s^4 \sqrt{p_1(\theta_2)T}\,.
\end{align*}
Thus, for any integer $m$, one has
\begin{eqnarray*}
  \mathrm{Q}_T(p_1(\theta_2)) 
 &\leq & \sum_{k=0}^{m-1}  7280 s^4 \sqrt{\left( \frac{1}{4s^2} \right)^k p_1(\theta_2)T} +  \mathrm{Q}_T \left( \left( \frac{1}{4s^2} \right)^m p_1(\theta_2) \right)  \\ 
 &\leq & \sum_{k=0}^{m-1}  \left( \frac{1}{2} \right)^k 7280 s^4 \sqrt{ p_1(\theta_2)T} +  \mathrm{R}_T \left(1- \left( \frac{1}{4s^2} \right)^m p_1(\theta_2) \right)  \\ 
 &\leq &  14560 s^4 \sqrt{ p_1(\theta_2)T} +  \left( \frac{1}{4s^2} \right)^m p_1(\theta_2) \Delta T \,,
\end{eqnarray*}
where we have used Lemma~\ref{lem:tech-2b2}(d) in the last step. Finally, letting $m$ go to infinity, we get \\ $\mathrm{Q}_T(p_1(\theta_2)) \leq 14560 s^4 \sqrt{ p_1(\theta_2)T}  $, that is, $\mathrm{R}_T(p_1(\theta_1)) \leq 14560 s^4 \sqrt{ (1-p_1(\theta_1))T}$.
\qed
\end{proof}

\section{Proof of \propref{pro:anchor}} \label{sec:anchor}

\begin{proof}[of \propref{pro:anchor}]
In this proof, we consider the case where $\theta_1$ is the true reward-generating model. We use the notation defined in Lemma~\ref{lem:stop}. First, the desired inequality is trivial if $p_1(\theta_1) \geq \frac{1}{3}$ since $\mathrm{R}_T(\cdot)$ is a decreasing function by \lemref{lem:decreasing_regret}. Let $p_1(\theta_1) \leq \frac{1}{3}$, $A=\frac{3}{2} p_1(\theta_1)$ and take $B>0$ such that $B \leq \frac{1}{2}p_1(\theta_1)$. The exact value of $B$ will be specified later. It is easy to see that $A \leq \frac{1}{2}$ and $B \leq \frac{1}{2} \leq 1-A$. We decompose the rest of the proof into three steps. 

\noindent\textbf{Step One:} This step is devoted to upper bounding $\E^{\theta_1}[\tau_A \wedge \tau_B -1]$. Note that by the definition of $\tau_A$ and $\tau_B$, one has for $t \leq \tau_A \wedge \tau_B -1$, $B \leq p_t(\theta_1) \leq A \leq \frac{1}{2}$ and $p_t(\theta_2) \geq 1-A \geq \frac{1}{2} \geq B $. Thus, by Lemma~\ref{lem:tech-2b2}(a), we have for $t \leq \tau_A \wedge \tau_B -1$,
\begin{eqnarray*}
 \lefteqn{\E_t^{\theta_1} \left[ \log(p_t(\theta_1)^{-1}) - \log(p_{t+1}(\theta_1)^{-1}) \right]} \\
& \geq &   \frac{1}{2} p_t(\theta_1)p_t(\theta_2)^2  \Delta_1^2 + \frac{1}{2} p_t(\theta_2)^3  \Delta_2^2   \\
&\geq & \frac{p_t(\theta_2)^2 B}{2}(\Delta_1^2 + \Delta_2^2)  \, \geq  \frac{B\Delta^2}{16}  \,,
\end{eqnarray*}
where we have used $\Delta_1^2+\Delta_2^2 \geq \frac{1}{2}(\Delta_1+\Delta_2)^2 \geq \frac{1}{2}\Delta^2$. Rearranging, we get for $t \leq \tau_A \wedge \tau_B -1$,
$$\E_t^{\theta_1} \left[ \log(p_{t+1}(\theta_1)^{-1}) + (t+1)\frac{B\Delta^2}{16} \right] 
  \leq \log(p_t(\theta_1)^{-1}) + t\frac{B\Delta^2}{16} .$$
In other words, $\left( \log(p_t(\theta_1)^{-1}) + t\frac{B\Delta^2}{16} \right)_{t \leq \tau_A \wedge \tau_B}$ is a supermartingale. 

Now, using Doob's optional stopping theorem, one has for any $t \geq 1$, 
$$\E^{\theta_1} \left[ \log(p_{t \wedge \tau_A \wedge \tau_B }(\theta_1)^{-1}) + (t \wedge \tau_A \wedge \tau_B)\frac{B\Delta^2}{16} \right]   \leq \log(p_1(\theta_1)^{-1}) + \frac{B\Delta^2}{16} .$$
Also, by Lemma~\ref{lem:smooth}, $\log(p_{t \wedge \tau_A \wedge \tau_B }(\theta_1)^{-1}) \leq \log\left(\frac{s}{B}\right)$ for any $t \leq 1$.  Using Lebesgue's dominated convergence theorem and the monotone convergence theorem, 
\begin{eqnarray*}
& &\E^{\theta_1} \left[ \log(p_{t \wedge \tau_A \wedge \tau_B }(\theta_1)^{-1}) + (t \wedge \tau_A \wedge \tau_B)\frac{B\Delta^2}{16} \right] \\
&&\longrightarrow \E^{\theta_1} \left[ \log(p_{\tau_A \wedge \tau_B }(\theta_1)^{-1}) + (\tau_A \wedge \tau_B)\frac{B\Delta^2}{16} \right]   \end{eqnarray*}
as $t \rightarrow +\infty$. Hence,
$$ \E^{\theta_1}[\tau_A \wedge \tau_B -1] \leq  \frac{16}{B\Delta^2} \E^{\theta_1} \left[ \log \frac{p_{\tau_A \wedge \tau_B }(\theta_1)}{p_1(\theta_1)}  \right]    \leq \frac{16}{B\Delta^2}  \log \frac{sA}{p_1(\theta_1)} 
=  \frac{16}{B\Delta^2}  \log \frac{3s}{2}  ,  $$
where we have used Lemma~\ref{lem:smooth} in the second last step. 

\noindent\textbf{Step Two:} In this step, we establish a recurrence inequality for the regret function $\mathrm{R}_T(\cdot)$. By \lemref{lem:markov}, $(p_t(\theta_1))_{t\geq 1}$ and  $(p_t(\theta_2))_{t\geq 1}$ are both Markov processes. Thus, the regret of Thompson Sampling can be decomposed as follows
\begin{eqnarray*}
\mathrm{R}_T(p_1(\theta_1))  &=& \Delta \cdot \E^{\theta_1}  \sum_{t=1}^T  p_t(\theta_2)  \\
   &=& \Delta \cdot \E^{\theta_1}  \sum_{t=1}^{\tau_A \wedge \tau_B -1}  p_t(\theta_2)
         + q_{B,A} \cdot  \E^{\theta_1} [  \mathrm{R}_T(p_{\tau_B}(\theta_1)) | \tau_A >\tau_B ]  \\
   &&  +\, q_{A,B} \cdot  \E^{\theta_1} [  \mathrm{R}_T(p_{\tau_A}(\theta_1)) | \tau_A <\tau_B ]   \\
   &\leq &  \Delta \cdot \E^{\theta_1}  [\tau_A \wedge \tau_B -1]
        + q_{B,A}  \Delta T  + \E^{\theta_1} [  \mathrm{R}_T(p_{\tau_A}(\theta_1)) | \tau_A <\tau_B ]   \\
   &\leq &  \frac{16}{B\Delta}  \log \frac{3s}{2}
        + \frac{B}{p_1(\theta_1)} \Delta T  +  \mathrm{R}_T \left( \frac{3}{2} p_1(\theta_1) \right) ,  
\end{eqnarray*} 
where in the last step, we have used the facts that $q_{B,A} \leq \frac{B}{p_1(\theta_1)}$ (by \lemref{lem:stop}), $p_{\tau_A}(\theta_1) \geq A =  \frac{3}{2} p_1(\theta_1) $, and $\mathrm{R}_T(\cdot)$ is a decreasing function (\lemref{lem:decreasing_regret}).

\noindent\textbf{Step Three:} The recurrence inequality established in the previous step and an appropriate choice of the parameter $B$ allow us to get the desired upper bound on $\mathrm{R}_T(p_1(\theta_1))$. 
On one side, if $\Delta \leq 2\sqrt{\frac{1}{p_1(\theta_1)T}}$, then $\mathrm{R}_T(p_1(\theta_1))\leq \Delta T \leq 2\sqrt{\frac{T}{p_1(\theta_1)}}$. On the other side, if $\Delta > 2\sqrt{\frac{1}{p_1(\theta_1)T}}$, we take $B=\frac{1}{\Delta} \sqrt{\frac{p_1(\theta_1)}{T}}$. This choice of $B$ is eligible since $\frac{1}{\Delta} \sqrt{\frac{p_1(\theta_1)}{T}} \leq \frac{1}{2} p_1(\theta_1)$. Then for any $p_1(\theta_1) \leq \frac{1}{3}$,
\begin{eqnarray*}
\mathrm{R}_T(p_1(\theta_1))  &\leq &  \frac{16}{B\Delta}  \log \frac{3s}{2}
        + \frac{B}{p_1(\theta_1)} \Delta T  +  \mathrm{R}_T \left( \frac{3}{2} p_1(\theta_1) \right) \\ 
 & = & \left( 16 \log \frac{3s}{2} + 1 \right) \sqrt{\frac{T}{p_1(\theta_1)}}
          + \mathrm{R}_T \left( \frac{3}{2} p_1(\theta_1) \right) .   
\end{eqnarray*}
It follows that for any integer  $h \geq 1$, as long as $\left(\frac{3}{2}\right)^{h-1} p_1(\theta_1) \leq \frac{1}{3}$, one has
\begin{eqnarray*}
\mathrm{R}_T(p_1(\theta_1)) &\leq & \sum_{k=0}^{h-1} \left( 16 \log \frac{3s}{2} + 1 \right) \sqrt{\left(\frac{2}{3}\right)^k\frac{T}{p_1(\theta_1)}} + \mathrm{R}_T \left( \left(\frac{3}{2}\right)^h p_1(\theta_1) \right) \\ 
   &\leq & \left(1-\sqrt{\frac{2}{3}}\right)^{-1}  \left( 16 \log \frac{3s}{2} + 1 \right) \sqrt{\frac{T}{p_1(\theta_1)}} + \mathrm{R}_T \left( \left(\frac{3}{2}\right)^h p_1(\theta_1) \right) \\ 
   &\leq &  \left( 96 \log \frac{3s}{2} + 6 \right) \sqrt{\frac{T}{p_1(\theta_1)}} + \mathrm{R}_T \left( \left(\frac{3}{2}\right)^h p_1(\theta_1) \right) .     
\end{eqnarray*}
Finally, by taking $h$ to be the smallest integer such that $\left(\frac{3}{2}\right)^h p_1(\theta_1) > \frac{1}{3}$ and using the fact that the function $\mathrm{R}_T(\cdot) $ is decreasing (\lemref{lem:decreasing_regret}), we get
$$\mathrm{R}_T(p_1(\theta_1)) \leq  \left( 96 \log \frac{3s}{2} + 6 \right) \sqrt{\frac{T}{p_1(\theta_1)}} + \mathrm{R}_T \left( \frac{1}{3} \right) , $$
which completes the proof. \qed
\end{proof}

\section{Proof of \propref{pro:anchor2}} \label{sec:anchor2}

\begin{proof}[of \propref{pro:anchor2}]
In this proof, we consider the case where $\theta_1$ is the true reward-generating model. We use the notation defined in Lemma~\ref{lem:stop}. Fix $T>0$ and $p_1(\theta_1) \leq \frac{1}{2}$. Let $B=\frac{1}{2} p_1(\theta_1)$ and take $A > p_1(\theta_1) $. The exact value of $A$ will be specified later. We decompose the proof into three steps. \\

\noindent\textbf{Step One:} This step is devoted to upper bounding $\E^{\theta_1}[\tau_A \wedge \tau_B -1]$. By \lemref{lem:tech-2b2}(c), we have for $t \leq \tau_A \wedge \tau_B -1$,
\begin{eqnarray*}
 \lefteqn{\E_t^{\theta_1} \left[ (1-p_{t+1}(\theta_1))^{-1} - (1-p_t(\theta_1))^{-1}  \right]} \\
 &\geq& \frac{p_{t}(\theta_1)^2}{2p_{t}(\theta_2)} \Delta_1^2 +  \frac{p_t(\theta_1)}{2} \Delta_2^2  \\
 &\geq& \frac{1}{2} B^2 \Delta_1^2 +  \frac{1}{2} B \Delta_2^2  \, \geq \, \frac{B^2 \Delta^2}{4} \,,
\end{eqnarray*}
where we have used $\Delta_1^2+\Delta_2^2 \geq \frac{1}{2}(\Delta_1+\Delta_2)^2 \geq \frac{1}{2}\Delta^2$. Rearranging, we get for $t \leq \tau_A \wedge \tau_B -1$,
$$ \E_t^{\theta_1} \left[ (1-p_{t+1}(\theta_1))^{-1} - (t+1) \frac{B^2 \Delta^2}{4} \right] \geq  (1-p_t(\theta_1))^{-1} - t \frac{B^2 \Delta^2}{4}\,.$$
In other words, $\left( (1-p_t(\theta_1))^{-1} - t \frac{B^2 \Delta^2}{4} \right)_{t \leq \tau_A \wedge \tau_B}$ is a submartingale.

Now, using Doob's optional stopping theorem, one has for any $t \geq 1$, 
$$ \E^{\theta_1} \left[ (1-p_{t \wedge \tau_A \wedge \tau_B}(\theta_1))^{-1} - (t \wedge \tau_A \wedge \tau_B) \frac{B^2 \Delta^2}{4} \right] \geq  (1-p_1(\theta_1))^{-1} -  \frac{B^2 \Delta^2}{4}.$$
Moreover, by Lemma~\ref{lem:smooth}, 
$$(1-p_{t \wedge \tau_A \wedge \tau_B}(\theta_1))^{-1} = p_{t \wedge \tau_A \wedge \tau_B}(\theta_2)^{-1}   \leq s \cdot p_{t \wedge \tau_A \wedge \tau_B -1}(\theta_2)^{-1} \leq \frac{s}{1-A} $$
for any $t \leq 1$.  Using Lebesgue's dominated convergence theorem and the monotone convergence theorem, 
\begin{eqnarray*}
&& \E^{\theta_1} \left[ (1-p_{t \wedge \tau_A \wedge \tau_B}(\theta_1))^{-1} - (t \wedge \tau_A \wedge \tau_B) \frac{B^2 \Delta^2}{4} \right] \\
&& \longrightarrow \E^{\theta_1} \left[ (1-p_{\tau_A \wedge \tau_B}(\theta_1))^{-1} - (\tau_A \wedge \tau_B) \frac{B^2 \Delta^2}{4} \right]
\end{eqnarray*}
as $t \rightarrow +\infty$. Hence,
\begin{eqnarray*}
\E^{\theta_1}[\tau_A \wedge \tau_B -1] &\leq& \frac{4}{B^2 \Delta^2} \E^{\theta_1} \left[ (1-p_{\tau_A \wedge \tau_B}(\theta_1))^{-1} \right] \\
&\leq& \frac{4s}{B^2 \Delta^2 (1-A)} \\
&=& \frac{16s}{p_1(\theta_1)^2 \Delta^2 (1-A)} \,.
\end{eqnarray*} 

\noindent\textbf{Step Two:} In this step, we establish a recurrence inequality for the regret function $\mathrm{R}_T(\cdot)$. By \lemref{lem:markov}, $(p_t(\theta_1))_{t\geq 1}$ and  $(p_t(\theta_2))_{t\geq 1}$ are both Markov processes. Thus, the regret of Thompson Sampling can be decomposed as follows
\begin{eqnarray*}
\mathrm{R}_T(p_1(\theta_1))  &=& \Delta \cdot \E^{\theta_1}  \sum_{t=1}^T  p_t(\theta_2)  \\
   &=& \Delta \cdot \E^{\theta_1}  \sum_{t=1}^{\tau_A \wedge \tau_B -1}  p_t(\theta_2)
         + q_{B,A} \cdot  \E^{\theta_1} [  \mathrm{R}_T(p_{\tau_B}(\theta_1)) | \tau_A >\tau_B ]  \\
   &&  +\, q_{A,B} \cdot  \E^{\theta_1} [  \mathrm{R}_T(p_{\tau_A}(\theta_1)) | \tau_A <\tau_B ]   \\
   &\leq &  \Delta  \E^{\theta_1}  [\tau_A \wedge \tau_B -1]
        + q_{B,A} \cdot  \E^{\theta_1} [  \mathrm{R}_T(p_{\tau_B}(\theta_1)) | \tau_A >\tau_B ] \\
  &&  +  \E^{\theta_1} [  \mathrm{R}_T(p_{\tau_A}(\theta_1)) | \tau_A <\tau_B ]     \\
   &\leq &  \frac{16s}{p_1(\theta_1)^2 \Delta (1-A)} 
        +  \frac{1}{2}  \mathrm{R}_T\left(\frac{1}{2s} p_1(\theta_1)\right)  +  (1-A)\Delta T   ,  
\end{eqnarray*} 
where in the last step, we have used the facts that  $q_{B,A} \leq \frac{B}{p_1(\theta_1)} = \frac{1}{2}$ (by Lemma~\ref{lem:stop}), $p_{\tau_B}(\theta_1) \geq \frac{B}{s} = \frac{1}{2s} p_1(\theta_1) $ (by Lemma~\ref{lem:smooth}),  $ \mathrm{R}_T(p_{\tau_A}(\theta_1)) \leq   (1-p_{\tau_A}(\theta_1))\Delta T \leq (1-A)\Delta T $ (by Lemma~\ref{lem:tech-2b2}(d))and $\mathrm{R}_T(\cdot)$ is a decreasing function (\lemref{lem:decreasing_regret})

\noindent\textbf{Step Three:} Finally, we establish the desired recurrence inequality by appropriately choosing the value of $A$. On one side, if $\Delta \leq \frac{2}{\sqrt{T}}$, then $\mathrm{R}_T(p_1(\theta_1))\leq \Delta T \leq 2\sqrt{T}$. On the other side, if $\Delta > \frac{2}{\sqrt{T}}$, we take $A=1-\frac{1}{\Delta\sqrt{T}}$. This choice of $A$ is eligible since $1-\frac{1}{\Delta\sqrt{T}} \geq \frac{1}{2} \geq p_1(\theta_1)$. Then for any $p_1(\theta_1) \leq \frac{1}{2}$,
\begin{eqnarray*}
\mathrm{R}_T(p_1(\theta_1))  &\leq & \frac{16s}{p_1(\theta_1)^2 \Delta (1-A)} 
        +  \frac{1}{2}  \mathrm{R}_T\left(\frac{1}{2s} p_1(\theta_1)\right)  +  (1-A)\Delta T   \\ 
 &\leq & \left(  \frac{16s}{p_1(\theta_1)^2} +1 \right) \sqrt{T}
          + \frac{1}{2}  \mathrm{R}_T\left(\frac{1}{2s} p_1(\theta_1)\right).   
\end{eqnarray*}
\qed
\end{proof} 

\section{Proof of \propref{pro:anchor3}} \label{sec:anchor3}

\begin{proof}[of \propref{pro:anchor3}]
In this proof, we consider the case where $\theta_1$ is the true reward-generating model. We use the notation defined in Lemma~\ref{lem:stop}. Fix $T>0$ and $p_1(\theta_1) \geq 1-\frac{1}{8s^2}$. Let $A=1-\frac{1}{4s^2}(1- p_1(\theta_1))$ and $B=1-4s(1- p_1(\theta_1))$. Then it is easy to see that $A > p_1(\theta_1) > B$ and $B \geq \frac{1}{2}$. The proof is decomposed into two steps.

\noindent\textbf{Step One:} This step is devoted to upper bounding $\E^{\theta_1}[\tau_A \wedge \tau_B -1]$. By \lemref{lem:tech-2b2}(c), we have for $t \leq \tau_A \wedge \tau_B -1$,
\begin{eqnarray*}
 \lefteqn{\E_t^{\theta_1} \left[ (1-p_{t+1}(\theta_1))^{-1} - (1-p_t(\theta_1))^{-1}  \right]} \\
 &\geq& \frac{p_{t}(\theta_1)^2}{2p_{t}(\theta_2)} \Delta_1^2 +  \frac{p_t(\theta_1)}{2} \Delta_2^2  \\
 &\geq&  \frac{1}{2} B^2 \Delta_1^2 + \frac{1}{2} B \Delta_2^2 \, \geq \, \frac{\Delta^2}{16} 
\end{eqnarray*}
where we have used $\Delta_1^2+\Delta_2^2 \geq \frac{1}{2}(\Delta_1+\Delta_2)^2 \geq \frac{1}{2}\Delta^2$. Rearranging, we get for $t \leq \tau_A \wedge \tau_B -1$,
$$ \E_t^{\theta_1} \left[ (1-p_{t+1}(\theta_1))^{-1} - (t+1)\frac{B^2}{16}  \right] \geq  (1-p_t(\theta_1))^{-1} - t\frac{B^2}{16} .$$
In other words, $\left( (1-p_t(\theta_1))^{-1} - t\frac{\Delta^2}{16}\right)_{t \leq \tau_A \wedge \tau_B}$ is a submartingale.

Now, using Doob's optional stopping theorem, one has for any $t \geq 1$, 
$$ \E^{\theta_1} \left[ (1-p_{t \wedge \tau_A \wedge \tau_B}(\theta_1))^{-1} - (t \wedge \tau_A \wedge \tau_B)\frac{\Delta^2}{16}  \right] \geq  (1-p_1(\theta_1))^{-1} - \frac{\Delta^2}{16} .$$
Moreover, by Lemma~\ref{lem:smooth}, 
$$(1-p_{t \wedge \tau_A \wedge \tau_B}(\theta_1))^{-1} = p_{t \wedge \tau_A \wedge \tau_B}(\theta_2)^{-1}   \leq s \cdot p_{t \wedge \tau_A \wedge \tau_B -1}(\theta_2)^{-1} \leq \frac{s}{1-A} $$
for any $t \geq 1$.  Using Lebesgue's dominated convergence theorem and the monotone convergence theorem, 
\begin{eqnarray*}
&& \E^{\theta_1} \left[ (1-p_{t \wedge \tau_A \wedge \tau_B}(\theta_1))^{-1} - (t \wedge \tau_A \wedge \tau_B)\frac{\Delta^2}{16}   \right] \\
&& \longrightarrow \E^{\theta_1} \left[ (1-p_{\tau_A \wedge \tau_B}(\theta_1))^{-1} - (\tau_A \wedge \tau_B)\frac{\Delta^2}{16}  \right]   \end{eqnarray*} 
as $t \rightarrow +\infty$. Hence,
$$ \E^{\theta_1}[\tau_A \wedge \tau_B -1] \leq \frac{16}{\Delta^2} \E^{\theta_1} \left[ (1-p_{\tau_A \wedge \tau_B}(\theta_1))^{-1} \right]  \leq  \frac{16s}{\Delta^2 (1-A)} . $$ 

\noindent\textbf{Step Two:} In this step, we establish the desired recurrence inequality.  By \lemref{lem:markov}, $(p_t(\theta_1))_{t\geq 1}$ and $(p_t(\theta_2))_{t\geq 1}$ are both Markov processes. Thus, the regret of Thompson Sampling can be decomposed as follows
\begin{eqnarray*}
\mathrm{R}_T(p_1(\theta_1))  &=& \Delta \cdot \E^{\theta_1}  \sum_{t=1}^T  p_t(\theta_2)  \\
   &=& \Delta \cdot \E^{\theta_1}  \sum_{t=1}^{\tau_A \wedge \tau_B -1}  p_t(\theta_2)
       +\, q_{A,B} \cdot  \E^{\theta_1} [  \mathrm{R}_T(p_{\tau_A}(\theta_1)) | \tau_A <\tau_B ]   \\ 
   &&  + \, q_{B,A} \cdot  \E^{\theta_1} [  \mathrm{R}_T(p_{\tau_B}(\theta_1)) | \tau_A >\tau_B ]  \\
   &\leq &  \Delta (1-B) \E^{\theta_1}  [\tau_A \wedge \tau_B -1] \\
  && + \, q_{A,B} \cdot  \mathrm{R}_T(A) + \, q_{B,A} \cdot   \mathrm{R}_T(1-s(1-B))  \\
   &\leq &  256 s^4 \sqrt{(1-p_1(\theta_1))T} + \, q_{A,B} \cdot  \mathrm{R}_T(1-\frac{1}{4s^2}(1- p_1(\theta_1))) \\
    & &  + \, q_{B,A} \cdot   \mathrm{R}_T\left(1-4s^2(1- p_1(\theta_1))\right) , 
\end{eqnarray*} 
where in last two steps, we have used the definition of $A, B$ and the facts that  $ p_{\tau_A}(\theta_1) \geq A$, $p_{\tau_B}(\theta_1) = 1- p_{\tau_B}(\theta_2)  \geq 1-s(1-B)$ (\lemref{lem:smooth}) and $\mathrm{R}_T(\cdot)$ is a decreasing function (\lemref{lem:decreasing_regret}).

Rearranging the newly obtained inequality, we get for $p_1(\theta_1) \geq 1-\frac{1}{8s^2}$,
\begin{eqnarray*}
 & &  \mathrm{R}_T(p_1(\theta_1)) - \mathrm{R}_T(1-\frac{1}{4s^2}(1- p_1(\theta_1)))  \\ 
 &\leq &  \frac{256s^4}{1-q_{B,A}} \sqrt{(1-p_1(\theta_1))T} \\
 && + \frac{q_{B,A}}{1-q_{B,A}} \left(  \mathrm{R}_T\left(1-4s^2(1- p_1(\theta_1)) \right) - \mathrm{R}_T(p_1(\theta_1)) \right) .
\end{eqnarray*}  
By Lemma~\ref{lem:stop}, 
$$q_{B,A} \leq \frac{1-p_1(\theta_1)}{A-B} = \frac{1}{4s-\frac{1}{4s^2}} \leq \frac{4}{15s} \leq  \frac{4}{15}.$$
Therefore, we obtain the desired recurrence inequality by observing that
$$ \frac{256s^4}{1-q_{B,A}} \leq \frac{3840s^4}{11} \leq 360 s^4\,,$$ and $$\frac{q_{B,A}}{1-q_{B,A}} \leq \frac{\frac{4}{15s}}{1-\frac{4}{15s}} \leq \frac{4}{11s}\,.$$
\qed
\end{proof}

\section{Experiments}
\label{sec:exp}

In this section, we give empirical evidence that the actual regret incurred by Thompson sampling is consistent with what theory predicts.  In particular, we show that the regret does indeed scale linearly with $\sqrt{1/p}$ and $\sqrt{1-p}$, respectively, for the good- and bad-prior cases, where $p$ is the prior probability mass of the true model.

We consider the \theNModel\ case, with Bernoulli rewards and $\theta_1$ being the true model.  The bandit problem is as described in Theorems~\ref{thm:lb-poor} and \ref{thm:lb-good} when $N=2$ with poor- and good-priors, respectively.  However, we make a more natural choice of fixing $\Delta=0.05$, as opposed to making it a function of $p$ and $T$ (required by the theorems).  Furthermore, when $N>2$, we introduce randomness into $\{\theta_2,\ldots,\theta_N\}$ to generate different models as follows.  For the poor-prior case, the reward is $0.5$ for $a=1$ and $0.5-\Delta'$ for $a=2$, where $\Delta'\sim\mathrm{Unif}[\frac{\Delta}{2},\frac{3\Delta}{2}]$; the good-prior case is constructed similarly.  Therefore, under every model other than $\theta_1$, the optimal action is $a=2$, whose per-step regret is $\Delta$ (since $\theta_1$ is actually the true model and $a=1$ is the true optimal action).

We place a prior probability mass $p>0$ to $\theta_1$, and assigns the rest of probability mass uniformly on the other $N-1$ models.  We run Thompson sampling with this prior, denoted $p_1$, for $T=10000$ steps; the cumulative regret over the $T$ steps is averaged over $2000$ independent runs of the algorithm, to yield a reliable empirical estimate of $\reg_T(\theta_1, \ts(p_1))$.  

\figref{fig:ts} shows the relation between $\reg_T(\theta_1, \ts(p_1))$ and $p$, for both the good- and bad-prior cases, with $N\in\{2,5\}$.  The $y$-axis is the average cumulative regret.  The left panel has $\sqrt{1/p}$ as the $x$-axis, with $p\in\{0.001, 0.002, 0.005, 0.01, 0.02, 0.05, 0.1\}$; the right panel has $\sqrt{1-p}$ as the $x$-axis, with $p\in\{0.995,$ $0.998,$  $0.999, 0.9995, 0.9998, 0.9999\}$.  As predicted by our upper/lower bounds, both plots show a scaling that is nearly linear, especially for the small-$p$ case.  For the large-$p$ case, the linear effect is more prominent when $p$ gets close to $1$ (that is, towards left end of the $x$-axis), as suggested by \thmref{thm:lb-good}.

More interestingly, we can see a similar scaling for $N=5$, although we have not provided the corresponding upper bounds in this work.  These empirical results indicate that the lower bounds in Corollary~\ref{cor:general_lb} may be tight, while the upper bound derived directly from previous results~\citep{Russo14Information} (see \secref{sec:comparison}) may not.

\begin{figure}
\centering
\begin{tabular}{cc}
\includegraphics[width=0.5\columnwidth]{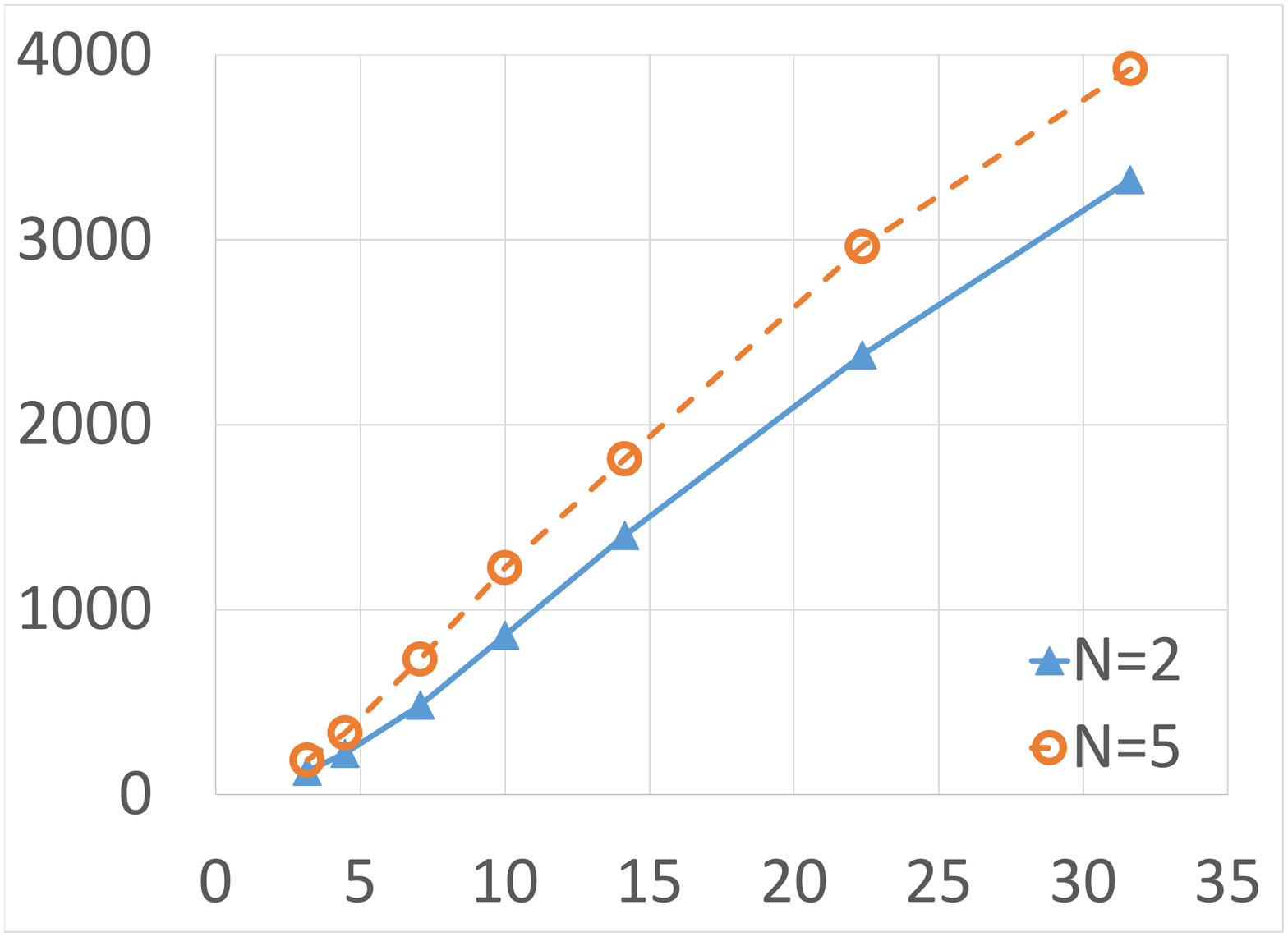} &
\includegraphics[width=0.5\columnwidth]{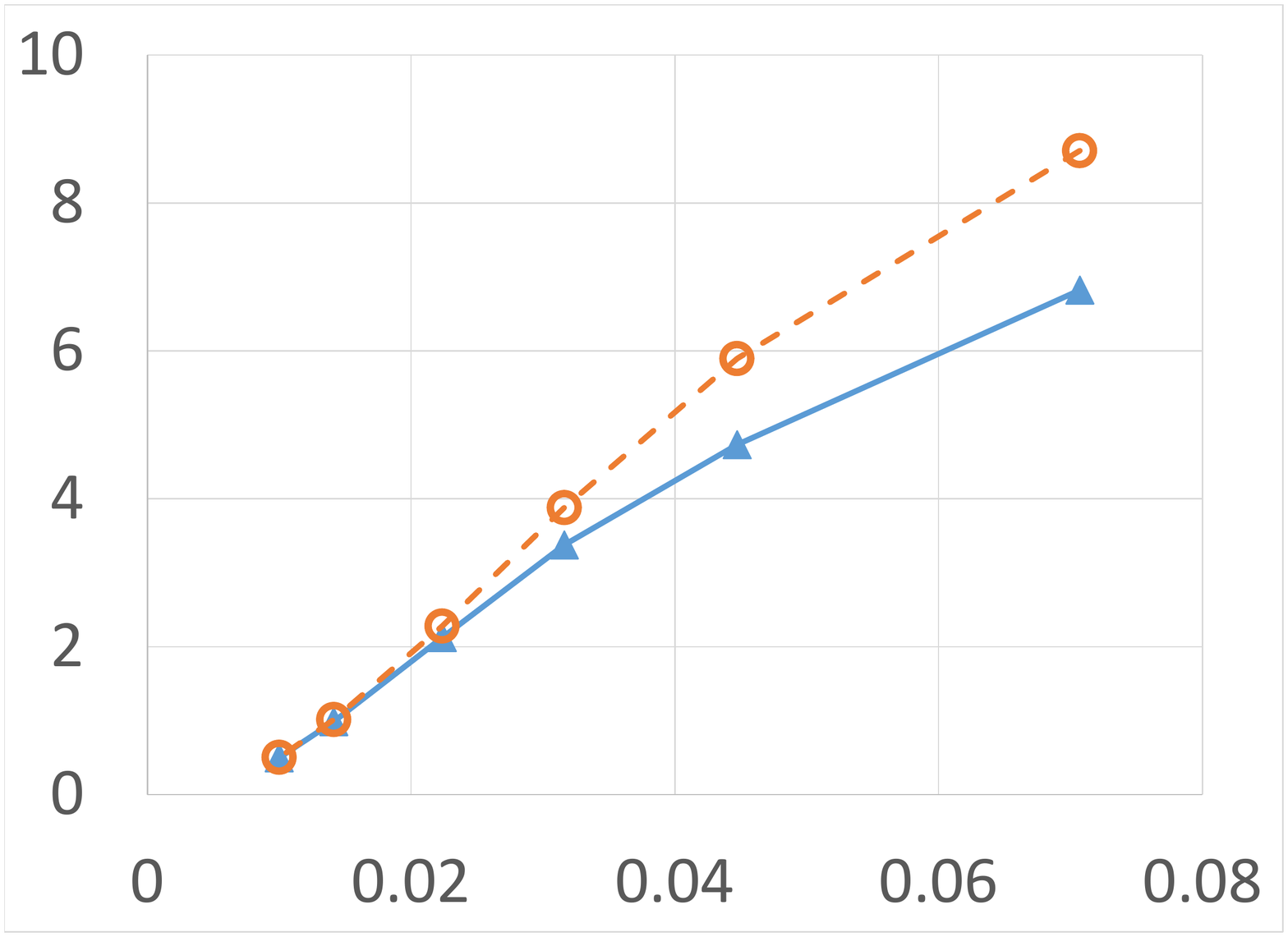}
\end{tabular}
\caption{Empirical cumulative regret $\reg_T(\theta_1, \ts(p_1))$, averaged over $2000$ runs, for the poor-prior (left) and good-prior (right) cases.  The $y$-axis is $\reg_T(\theta_1, \ts(p_1))$.  The $x$-axis is $\sqrt{1/p}$ for the left panel, and $\sqrt{1-p}$ for the right. \todo{TODO: (1) update with newer results; (2) error bars}} \label{fig:ts}
\end{figure}

\end{document}